%% file: mainfile_dr_camera.tex
\documentclass{article}

 \usepackage[nonatbib,final]{nips_2017}

\newcommand{\bibpath}{bib_nonmonotone}

\input{defs_jmlr.tex}

\allowdisplaybreaks

\usepackage[utf8]{inputenc} 
\usepackage[T1]{fontenc}    

\usepackage{url}            
\usepackage{booktabs}       

\usepackage{amsfonts}       
\usepackage{nicefrac}       
\usepackage{microtype}      
\usepackage{wrapfig}

\usepackage{thm-restate}

\graphicspath{{./figs/}}

\usepackage{enumitem}
\setlist[itemize]{leftmargin=9mm}

\title{
Continuous DR-submodular  Maximization:
 Structure and Algorithms
}

\author{
  An Bian\thanks{Now known as Yatao A. Bian. ORCID: \href{https://orcid.org/0000-0002-2368-4084}{orcid.org/0000-0002-2368-4084}}\\
  ETH Zurich\\
  \texttt{ybian@inf.ethz.ch}
  \\
  \And
  Kfir Y. Levy\\
  ETH Zurich\\
  \texttt{yehuda.levy@inf.ethz.ch}\\  
  \AND
  Andreas Krause\\
  ETH Zurich\\
  \texttt{krausea@ethz.ch} \\
  \And
    Joachim M. Buhmann\\
    ETH Zurich\\
    \texttt{jbuhmann@inf.ethz.ch} \\
}

\begin{document}
\renewcommand{\b}{\bm{b}}
\renewcommand{\v}{\bm{v}}
\renewcommand{\u}{\bm{u}}
\renewcommand{\r}{\bm{r}}
\renewcommand{\d}{\bm{d}}
\renewcommand{\c}{\bm{c}}
\renewcommand{\O}{\mathcal{O}}

\maketitle

\input{abs}

\input{intro}

%

\input{examples}

\input{structures}

\input{algs}


\input{exp}

\input{related}

\input{disc}

\clearpage
{
\bibliography{\bibpath}
}

\clearpage 
\appendix
\appendixtitle{Appendix}

\input{lattice}

\input{appendix}

\end{document}

%% file: defs_jmlr.tex
\usepackage[usenames]{xcolor}
\usepackage[bookmarksnumbered=true]{hyperref}
\usepackage{url}
\definecolor{cite_color}{RGB}{0, 0,255}
\definecolor{link_color}{RGB}{153, 0,0}  
\definecolor{url_color}{RGB}{153, 102,  0}
\definecolor{emp_color}{RGB}{0,0,255}
\hypersetup{
 colorlinks,
 citecolor=cite_color,
 linkcolor=link_color,
 urlcolor=url_color}

\usepackage{epsfig}
\usepackage{graphicx}
\usepackage{amsmath}
\usepackage{amssymb}
\usepackage{graphicx}
\usepackage{subfig}
\usepackage{multirow}
\usepackage{booktabs}
\usepackage{wrapfig}
\usepackage{enumerate}
\usepackage{bm}
\usepackage{tabularx}

\usepackage[numbers]{natbib}
\bibpunct{[}{]}{;}{n}{,}{,}
\renewcommand{\citet}{\cite}

\usepackage{mathtools}

\usepackage[vlined, ruled, linesnumbered]{algorithm2e}

\SetCommentSty{mycommfont}
\SetKwInOut{Input}{input}
\SetKwInOut{Output}{output}

 \usepackage{cleveref} 
 \crefname{section}{Section}{Sections}
 \crefname{theorem}{Theorem}{Theorems}
 \crefname{lemma}{Lemma}{Lemmas}
 \crefname{equation}{Equation}{Equations}
 \crefname{proposition}{Proposition}{Propositions}
 \crefname{claim}{Claim}{Claims}
  \crefname{corollary}{Corollary}{Corollaries}
   \crefname{algorithm}{Algorithm}{Algorithms}
 \crefname{figure}{Figure}{Figures}
 \crefname{table}{Table}{Tables}
 \crefname{remark}{Remark}{Remarks}
 \crefname{definition}{Definition}{Definitions}
  \crefname{appendix}{Appendix}{Appendices}

\newcommand{\appendixtitle}[1]{
	\begin{center}
		\LARGE \bf #1
	\end{center}
}

\def\X{{\cal X}}

\def\cone{{\cal K}}

\def \c{\bm{c}}
\def \v{\bm{v}}

\def \r{\bm{r}}

\def \c{\bm{c}}
\def \d{\bm{d}}
\def \a{\bm{a}}
\def \x{\bm{x}}
\def \y{\bm{y}}

\def \e{\bm{e}}
\def \u{\bm{u}}
\def \bmu{\bm{u}}
\def \bmA{\mathbf{A}}

\def \bmX{\mathbf{X}}
\def \bmY{\mathbf{Y}}
\def \bmI{\mathbf{I}}
\def \bmC{\mathbf{C}}
\def \bmD{\mathbf{D}}

\def \bmH{\mathbf{H}}
\def \bmL{\mathbf{L}}

\def \bmU{\mathbf{U}}

\def \bmW{\mathbf{W}}
\def \bmZ{\mathbf{Z}}

\def \bmtheta{\bm{\theta}}

\def \bmalpha{\bm\alpha}

\def \O{{\mathcal{O}}}
\def \P{{\mathcal{P}}}
\def \Q{{\mathcal{Q}}}

\def \R{{\mathbb{R}}}
\def \trans{\top}
\newcommand{\tr}[1]{\text{tr}(#1)}
\newcommand{\pare}[1]{{(#1)}}  

\def \z{\bm{z}}
\def \h{\bm{h}}

\newcommand{\argmin}{{\arg\min}}
\newcommand{\diag}{{\text{diag}}}

\newcommand{\sign}[1]{{\text{sign}(#1)}}
\newcommand{\argmax}{{\arg\max}}
\newcommand{\algname}[1]{{\textsc{#1}}}

\newcommand{\dtp}[2]{\langle #1, #2\rangle}
\newcommand{\de}[1]{\text{det}\left(#1\right)}
\newcommand{\fracpartial}[2]{\frac{\partial #1}{\partial #2}}
\newcommand{\fracppartial}[3]{\frac{\partial^2 #1}{\partial #2 \partial #3}}

\newcommand{\lleq}{\preceq}
\newcommand{\ggeq}{\succeq}

\newcommand{\bas}{\bm{e}} 

\newcommand{\groundset}{\ensuremath{\mathcal{V}}}

\usepackage{amsthm}

\newtheorem{theorem}{Theorem}
\newtheorem{lemma}{Lemma}
\newtheorem{proposition}{Proposition}
\newtheorem{corollary}{Corollary}
\newtheorem{definition}{Definition}
\newtheorem{remark}{Remark}
\newtheorem{claim}{Claim}

\usepackage{mathtools}

\DeclarePairedDelimiter\floor{\lfloor}{\rfloor}

%% file: abs.tex
\begin{abstract}

DR-submodular continuous functions are important objectives with wide real-world applications spanning MAP inference
in  determinantal point processes (DPPs), and mean-field inference for probabilistic submodular models, amongst others. 
DR-submodularity
captures a subclass of non-convex functions
that enables both exact minimization and approximate maximization in polynomial time. 

In this work we study the  problem of 
maximizing  \textit{non-monotone} continuous DR-submodular  functions
under general   down-closed convex constraints. 
We start by investigating geometric properties that underlie such objectives, e.g., a strong relation between  (approximately) stationary
points and global optimum is proved. These properties are then  used to devise two  optimization algorithms with provable guarantees.
Concretely, we first devise a ``two-phase'' algorithm with $1/4$ 
approximation guarantee. This algorithm allows the use of existing methods for finding (approximately) stationary points as a subroutine, thus, harnessing  recent progress in non-convex optimization. 
Then  we present 
a non-monotone \algname{Frank-Wolfe} variant
with $1/e$ approximation guarantee and sublinear
convergence rate. 
Finally, we extend our approach to a broader class of generalized DR-submodular continuous functions, which captures a wider spectrum of applications.
Our theoretical findings are validated on  synthetic and real-world problem instances.

%
%
%

\end{abstract}

%% file: intro.tex
\section{Introduction}
\label{sec_intro}


Submodularity is classically most well known for set function
optimization, where it
%
enables efficient minimization \citep{iwata2001combinatorial} and
approximate maximization
\citep{nemhauser1978analysis,krause2012submodular} in polynomial time.
Submodularity has recently been studied on the integer lattice
\citep{soma2014optimal,DBLP:conf/nips/SomaY15} and on continuous domains
\citep{bach2015submodular,bian2017guaranteed,staibrobust,hassani2017gradient}, with
significant theoretical results and practical applications. For set
functions, it is well known that submodularity is equivalent to the
diminishing returns (DR) property. However, this does not hold for
integer-lattice functions or continuous functions, where the DR
property defines a subclass of submodular functions, called 
DR-submodular functions.

In continuous domains, 
 applying convex
optimization techniques enables  efficient minimization of
submodular continuous functions \citep{bach2015submodular,staibrobust}  (despite the non-convex nature of such
objectives).
In \citet{bian2017guaranteed} it is further shown that continuous
submodularity enables constant-factor approximation schemes for
constrained monotone DR-submodular maximization and ``box''
constrained non-monotone submodular maximization problems.

Many real-world non-convex problems, such as maximizing the softmax
extension of DPPs, require maximizing a \emph{non-monotone}
DR-submodular function over a \emph{general} down-closed {convex}
constraint.  Yet, current theory
\citep{bach2015submodular,bian2017guaranteed,staibrobust} does not
apply to this general problem setting, which motivates us to 
 develop guaranteed and efficient algorithms for such problems.

Exploring the structure that underlies DR-submodularity is crucial to deriving
guaranteed algorithms.
Combined with a notion of
non-stationarity for constrained optimization problems and a new
notion of ``strong DR-submodularity'', we find  a rich
structure in the problem of continuous DR-submodular maximization. This in turn gives rise to two approximation algorithms
with provable guarantees.
%
Specifically, we make the following contributions:

%
%
%
\begin{itemize}

\item We bound the difference between objective values of stationary
  points and the global optimum. Our analysis shows that the
  bound is even tighter if the objective is strongly
  DR-submodular (see \cref{eq_strongly_dr}). 

%
%

\item Based on the geometric  properties, we present two
  algorithms: (i) A two-phase \algname{Frank-Wolfe}-style algorithm with
  $1/4$ approximation guarantee converges with a $1/\sqrt{k}$ rate;
  (ii) a non-monotone \algname{Frank-Wolfe} variant exhibits a $1/e$
  approximation guarantee and converges sublinearly. Even though the
  worst-case guarantee of the first one is worse than the second, it
  yields several practical advantages, which we
  discuss in \cref{sec_fw_variant}.
	
\item We investigate a generalized class of submodular functions on
  ``conic'' lattices.  This allows us to model a larger class of
  non-trivial applications. These include logistic regression with a
  non-convex separable regularizer, non-negative PCA, etc.
  To optimize them, we provide a reduction that enables to invoke
  algorithms for continuous submodular optimization problems.

\item We experimentally demonstrate the applicability of our methods
  on both synthetic and real-world problem instances.
	
\end{itemize}

%
%
%
%

\subsection{Problem Statement}

\textbf{Notation.}
We use boldface letters, e.g., $\x$ to represent a
vector,  boldface capital letters, e.g., $\bmA$  to denote
a matrix. $x_i$ is the $i^{\text{th}}$ entry of $\x$, $A_{ij}$ is the $(ij)^{\text{th}}$ entry of $\bmA$.  We use  $\bas_i$ to denote the standard $i^\text{th}$ basis vector. 
 $f(\cdot)$ is used  to denote a continuous function, and $F(\cdot)$ to
represent a set function. 
$[n]:= \{1,...,n\}$ for an integer $n \geq 1$.
$\|\cdot\|$ means the Euclidean norm by default.  Given two vectors
$\x,\y$, $\x\leq \y$ means $x_i\leq y_i, \forall i$.  $\x\vee \y$ and
$\x \wedge \y $ denote coordinate-wise maximum and coordinate-wise
minimum, respectively.

The general setup of constrained non-monotone DR-submodular (see
\cref{ded:DR_sub} below) maximization is,
\begin{align}\label{setup}
\max_{\x \in \P} f(\x),     \tag{P}
\end{align}
where $f: \X \rightarrow \R$ is continuous DR-submodular,
$\X = \prod_{i=1}^{n}\X_i$, each $\X_i$ is an interval
\citep{bach2015submodular,bian2017guaranteed}.
%
Wlog\footnote{Since otherwise one can work on a new function
  $g(\x) := f(\x + \underline \u)$ that has ${0}$ as the lower
  bound of its domain, and all properties of the function are still
  preserved.}, we assume that the lower bound $\underline \u$ of $\X$
is ${0}$, i.e., $\X=[0, \bar \bmu ]$.  The set
$\P\subseteq [0, \bar \bmu]$ is assumed to be a down-closed {convex}
set, where down-closedness means: $\x \in \P$ and $0\leq\y \leq \x$
implies that $\y \in \P$.
The diameter of $\P$ is $D:= \max_{\x,\y\in\P}\|\x-\y\|$, and it holds
that $D \leq \|\bar \u \|$.  We use $\x^*$ to denote the
global maximum of \labelcref{setup}.
One can assume $f$ is non-negative over $\X$, since otherwise one just
needs to find a lower bound for the minimum function value of $f$ over
$[0, \bar \u]$ (and box-constrained submodular minimization can be
solved to arbitrary precision in polynomial time
\citep{bach2015submodular}).
%
Over continuous domains, a DR-submodular function
\citep{bian2017guaranteed} is a submodular function with the
diminishing returns (DR) property,
\begin{definition}[DR-submodular \& DR property]
	\label{ded:DR_sub} 
	A  function $f:\X\mapsto \R$ is  DR-submodular (has the DR property) if $\forall \a\leq \b \in \X$, $\forall i \in [n], \forall k\in \R_+$ s.t. $(k\e_i+\a)$ and $(k\e_i+\b)$ are still in $\X$, it holds,
	\begin{align}\label{eq_dr}
	f(k\e_i+\a) - f(\a) \geq f(k\e_i+\b) - f(\b).
	\end{align}
\end{definition}
If $f$ is differentiable, one can show that  \cref{ded:DR_sub} is equivalent to  $\nabla f$ being
an antitone mapping from $\R^n$ to $\R^n$. Furthermore, if $f$ is twice-differentiable, 
the DR property is equivalent to  all of the entries
of its Hessian being non-positive, i.e., $\nabla^2_{ij} f(\x) \leq 0, \forall \x\in \X, i,j\in [n]$.  A function $f:\X\mapsto \R$ is DR-supermodular 
iff $-f$ is DR-submodular. 
We also assume that $f$ has 
 Lipschitz gradients, 
%
\begin{definition}
A  function $f$ has $L$-Lipschitz gradients if for all $\x,\y \in \X$  it holds that,
\begin{align}\label{eq_smooth}
\| \nabla f(\x)- \nabla f(\y) \| \leq L \|\x - \y\|.
\end{align}
\end{definition}
A brief summary of related work appears in \cref{sec_related}.

%


%% file: examples.tex
\section{Motivating Real-world Examples}
\label{sec_examples}
Many continuous objectives in practice turn out to be DR-submodular.
Here we  list several  of them. More 
can be found in \cref{sec_more_apps}.

\textbf{Softmax extension.}  
Determinantal point processes (DPPs) are  probabilistic models
of repulsion, that have been used to model diversity in machine
learning \citep{kulesza2012determinantal}. The constrained MAP
(maximum a posteriori) inference problem of a DPP is an NP-hard
combinatorial problem in general. Currently, the methods with the best
approximation guarantees are based on either maximizing the
multilinear extension \citep{calinescu2007maximizing} or the softmax
extension \citep{gillenwater2012near}, both of which are DR-submodular
functions (details in \cref{appe_dr_soft}).
The multilinear extension is given as an expectation over the original
set function values, thus evaluating the objective of this extension
requires expensive sampling.  In constast, the softmax extension has a
closed form expression, which is much more appealing from a
computational perspective.
%
Let $\bmL$ be the positive semidefinite kernel matrix of a DPP, its
softmax extension is:
\begin{flalign}\label{eq_softmax}
  f(\x) = \log\de{\diag(\x)(\bmL-\bmI) +\bmI }, \x\in [0,1]^n,
\end{flalign}
where $\bmI$ is the identity matrix, $\diag(\x)$ is the diagonal
matrix with diagonal elements set as $\x$. The problem of MAP
inference in DPPs corresponds to the problem $\max_{\x\in \P} f(\x)$,
where $\P$ is a down-closed convex constraint, e.g., a matroid
polytope or a matching polytope.

\textbf{Mean-field inference for log-submodular models.}
Log-submodular models \citep{djolonga14variational} are a class of
probabilistic models over subsets of a ground set $\groundset = [n]$,
where the log-densities are submodular set functions $F(S)$:
$p(S) = \frac{1}{Z}\exp(F(S))$. The partition function
$Z = \sum_{S\subseteq \groundset}\exp(F(S))$ is typically hard to
evaluate.  One can use mean-field inference to approximate $p(S)$ by
some factorized distribution
$q_{\x}(S):= \prod_{i\in S}x_i \prod_{j\notin S}(1-x_j), \x\in
[0,1]^n$,
by minimizing the distance measured w.r.t. the Kullback-Leibler
divergence between $q_{\x}$ and $p$, i.e.,
$ \sum_{S\subseteq \groundset} q_{\x}(S)
\log\frac{q_{\x}(S)}{p(S)}$. It is,
\begin{align}\notag 
  \text{KL}(\x) = 
  -\sum_{S\subseteq \groundset}\prod_{i\in S}x_i \prod_{j\notin S}(1-x_j) F(S) + \sum\nolimits_{i=1}^{n} [x_i\log x_i + (1-x_i)\log(1-x_i)] + \log Z.
\end{align}
$ \text{KL}(\x)$ is DR-supermodular w.r.t. $\x$ (details in
\cref{appe_dr_soft}).  Minimizing the Kullback-Leibler divergence
$\text{KL}(\x)$ amounts to maximizing a DR-submodular function.

\subsection{Motivating  Example Captured by  Generalized Submodularity  on Conic Lattices}
\label{subsec_moti_lattice}

Submodular continuous functions can already
model many scenarios. Yet, there are 
several interesting cases which are in general not (DR-)Submodular, but  can still be captured by a generalized notion. This generalization enables to
develop polynomial algorithms with   guarantees by using ideas from continuous submodular  optimization. We  present one representative objective here  (more in \cref{sec_more_apps}). In \cref{sec_lattice} we show the technical details on    how  they are  covered by
a class of submodular continuous functions over
conic lattices.

Consider the   logistic regression model with a \emph{non-convex} separable regularizer.
This flexibility may result in   better statistical  performance (e.g., in recovering
discontinuities, \citep{antoniadis2011penalized}) compared to classical models with convex regularizers.
%
Let  $\z^1,..., \z^m$ in $\R^n$ be $m$ training  points  with corresponding binary labels $\y\in \{\pm 1\}^m$. Assume that the following mild assumption is satisfied:  
	For any  fixed dimension $i$,
	all the data points  have the same sign, i.e., 
	$\sign{z^j_i}$ is the same for all $j \in [m]$ (which can
	be achieved by easily scaling if not).  
 The task  is to solve the following non-convex optimization problem, 
\begin{flalign}\label{lr}
\min_{\x\in \R^n} f(\x) := m^{-1}\sum\nolimits_{j=1}^{m}f_j(\x)  +\lambda  r(\x), 
\end{flalign}
where $f_j(\x) = \log(1 +\exp(-y_j \x^\trans \z^j))$ is the logistic loss; $\lambda>0$ is the regularization parameter, and  $r(\x) $ is  some non-convex separable regularizer. 
Such separable regularizers are popular in statistics,  and two notable choices are $r(\x)=  \sum_{i= 1}^n \frac{\gamma x_i^2}{1+\gamma  x_i^2}$, and $r(\x) =  \sum_{i= 1}^n \min \{\gamma x_i^2, 1 \}$ (see ~\citep{antoniadis2011penalized}). 
Let us define a vector $\bmalpha\in \{\pm 1 \}^n$ as $\alpha_i = \text{sign}(z_i^j), i\in [n]$ and $l(\x) :=\frac{1}{m}\sum\nolimits_{j=1}^{m}f_j(\x)$.
One can show that $l(\x)$ is not DR-submodular or DR-supermodular. 
Yet,  in \cref{sec_lattice} we  will show that  $l(\x)$  is   $\cone_{\bmalpha}$-DR-supermodular, where the latter generalizes DR-supermodularity.
Usually, one can assume  the optimal solution $\x^*$ lies in 
some  box $[\underline \u, \bar \u]$. Then the problem is an instance of constrained  non-monotone  
$\cone_{\bmalpha}$-DR-submodular maximization.


%% file: structures.tex
\vspace{-0.15cm}
\section{Underlying  Properties of  Constrained  DR-submodular  Maximization
} 
\label{sec_structures}



\vspace{-0.1cm}
In this section we  present several properties arising in DR-submodular function maximization. First we show properties related to    concavity  of the objective along certain directions,  then we establish  the  
relation between  locally 
stationary points and the global
optimum (thus called ``local-global relation''). 
These properties will be used to derive guarantees for the algorithms in \cref{sec_algs}. All omitted proofs are in \cref{app_proofs_struc_algs}.

\subsection{Properties Along Non-negative/Non-positive Directions}

A DR-submodular function $f$  is \emph{concave} along any non-negative/non-positive direction \citep{bian2017guaranteed}.
%
Notice that DR-submodularity is a stronger
condition than concavity along directions $\v \in \pm \R_+^n$: for instance,  a concave function is concave 
along any  direction,
but it may not be a DR-submodular function.   

For a  DR-submodular function with $L$-Lipschitz gradients, one can get the following quadratic lower bound using 
standard techniques by combing the concavity and Lipschitz 
gradients in \labelcref{eq_smooth}. \\
\textbf{Quadratic lower bound.}
If $f$ is  DR-submodular  with a $L$-Lipschitz gradient,  then for all $\x \in \X$ and $\v \in \pm \R_+^n$, it holds, 
\begin{align}\label{eq_quad_lower_bound}
f(\x + \v)\geq f(\x) + \dtp{\nabla f(\x)}{\v} - \frac{L}{2}\|\v\|^2.
\end{align}
It  will be used in \cref{sec_fw_variant} for 
analyzing the non-monotone \algname{Frank-Wolfe} variant (\cref{fw-non-monotone}).

\textbf{Strong DR-submodularity.} 
 DR-submodular objectives may be 
 strongly concave along   directions $\v \in \pm \R_+^n$, e.g., for  DR-submodular quadratic functions. 
 We will show that such additional structure may be exploited to obtain stronger guarantees for the local-global relation.
\begin{definition}[Strongly DR-submodular] \label{eq_strongly_dr}
	A function $f$ is $\mu$-strongly 
DR-submodular ($\mu\geq 0$) if for all $\x\in \X$ and  $\v \in \pm \R_+^n$, it holds that,
\begin{align}\label{eq_strong_dr}
f(\x+\v) \leq f(\x) + \dtp{\nabla f(\x)}{\v} - \frac{\mu }{2}\|\v\|^2.
\end{align}
\end{definition}

\subsection{Relation Between Approximately Stationary Points and Global Optimum}\label{subsec_local_global}

First of all, we present the following \namecref{lemma_3_1}, which
will motivate us to consider a  non-stationarity measure
for general constrained optimization problems. 
\begin{lemma}\label{lemma_3_1}
If $f$ is $\mu$-strongly DR-submodular, then for any two points $\x$,  $\y$ in $\X$, it holds: 
\begin{align}\label{non_stationarity}
(\y-\x)^{\trans}\nabla f(\x) \geq f(\x\vee\y) + f(\x\wedge \y) - 2f(\x) + \frac{\mu}{2}\|\x -\y\|^2.
\end{align}
\end{lemma}  
 \cref{lemma_3_1} implies that if $\x$ is  stationary (i.e., $\nabla f(\x)=0$), then $2f(\x) \geq f(\x\vee\y) + f(\x\wedge \y)  + \frac{\mu}{2}\|\x -\y\|^2$, which gives an implicit relation
between $\x$ and  $\y$. While in practice finding 
an exact  stationary point is   not easy, usually  non-convex solvers will
 arrive at an approximately stationary point, thus requiring a
proper measure of non-stationarity for the constrained optimization problem.

\textbf{Non-stationarity measure.}
 Looking at the LHS of \labelcref{non_stationarity}, it naturally suggests to use 
$\max_{\y\in\P}(\y-\x)^{\trans}\nabla f(\x)$ as the non-stationarity measure,
which happens to coincide with the measure proposed by  recent work
of  \citet{lacoste2016convergence}, and it  can
be calculated for free for  {Frank-Wolfe}-style algorithms (e.g.,  \cref{nonconvex_fw}). In order to adapt it to the local-global relation,
we give a slightly more general definition here: For any   constraint set $\Q\subseteq \X$, the non-stationarity 
of a point $\x\in \Q$ is, 
\begin{align}\label{non_stationary}
 g_{\Q}(\x) := \max_{\v\in\Q}\dtp{\v - \x}{\nabla f(\x)} \qquad \text{(non-stationarity)}.
\end{align}
It always holds that $g_{\Q}(\x)\geq 0$, and $\x$ is a
stationary point in $\Q$ iff $g_{\Q}(\x)=0$, so \labelcref{non_stationary} is a natural generalization of the non-stationarity measure
$\|\nabla f(\x)\|$ for unconstrained  optimization. As the next statement shows, $g_{\Q}(\x)$ 
plays an important role in characterizing the local-global 
relation.


\begin{proposition}[Local-Global Relation]\label{local_global}
Let $\x$ be  a point  in $\P$
with non-stationarity  $g_{\P}(\x)$, and   
${\Q} := \{\y \in \P \;|\; \y\leq \bar \u - \x\}$. 
Let  $\z$ be  a point  in $\Q$ 
with non-stationarity  $g_{\Q}(\z)$.
It holds that,
\begin{flalign}
\max\{f(\x), f(\z) \}  \geq \frac{1}{4}\left[f(\x^*) -g_{\P}(\x) -g_{\Q}(\z)\right ]  +   \frac{\mu}{8}\left(\|\x -\x^*\|^2 + \|\z -\z^*\|^2\right ),
\end{flalign}
where $\z^*:= \x\vee \x^* -\x$.
\end{proposition}
\textbf{Proof sketch of \cref{local_global}:}
The proof uses \cref{lemma_3_1}, the non-stationarity in \labelcref{non_stationary} and a key observation in the following 
\namecref{claim_key}. 
The detailed proof is in \cref{app_claim_proof}.
\begin{restatable}[]{claim}{keyclaim}
\label{claim_key}
It holds that $f(\x\vee \x^*) + f(\x \wedge \x^*) +  f(\z\vee \z^*) + f(\z \wedge \z^*) \geq f(\x^*)$.
\end{restatable}
Note that 
\citet{chekuri2014submodular,gillenwater2012near} propose a similar   relation 
for the special cases of multilinear/softmax extensions by mainly proving 
the same conclusion as in \cref{claim_key}. Their relation  does not   incorporate the properties of 
non-stationarity or strong DR-submodularity. 
They both use the proof idea of constructing a complicated
auxiliary set  function tailored to specific DR-submodular functions.
We 
present a different  proof method by directly 
utilizing the   DR property on carefully constructed auxiliary points (e.g., $(\x+\z)\vee \x^*$ in the proof of \cref{claim_key}), 
 this is arguably more succint and straightforward than that 
of \citet{chekuri2014submodular,gillenwater2012near}.

%% file: algs.tex
\section{Algorithms for Constrained DR-submodular Maximization}
\label{sec_algs}


Based on the properties,  we present two algorithms for
solving    \labelcref{setup}.
The first  is based on the local-global relation,  and the second  is a \algname{Frank-Wolfe} variant adapted for the non-monotone 
setting.  All  the omitted proofs are deferred to \cref{app_proofs_algs}. 

\subsection{An Algorithm Based on the Local-Global Relation}
\label{subsec_local_alg}

\IncMargin{1em}
\begin{algorithm}[htbp]
	\caption{\algname{two-phase Frank-Wolfe} for non-monotone
		{DR}-submodular
		maximization}\label{lg_fw}
 
	\KwIn{$\max_{\x \in \P} f(\x)$,
stopping tolerance $\epsilon_1, \epsilon_2$, \#iterations $K_1, K_2$}
 
	{$\x \leftarrow \algname{Non-convex Frank-Wolfe}(f, \P, K_1, \epsilon_1, \x^\pare{0}) $    \tcp*{$\x^\pare{0}\in \P$}}
	{$\Q \leftarrow  \P\cap \{\y\in \R_+^n\;|\;\y\leq \bar \u -\x \}$\;}
	{  $\z \leftarrow \algname{Non-convex Frank-Wolfe}(f, \Q, K_2, \epsilon_2, \z^\pare{0}) $    \tcp*{$\z^\pare{0}\in \Q$}}
 
	\KwOut{$\argmax \{f(\x), f(\z)\}$ \;}
\end{algorithm}
\DecMargin{1em}

We summarize the \algname{two-phase} algorithm in 
\cref{lg_fw}. It is generalized from the ``two-phase'' method
in \citet{chekuri2014submodular,gillenwater2012near}. It invokes some non-convex
solver (we use the \algname{non-convex Frank-Wolfe} by \citet{lacoste2016convergence}; pseudocode is included in   \cref{nonconvex_fw} of  \cref{sec_nonconvex_fw}) to find  approximately stationary
points in $\P$ and $\Q$,  respectively, then returns
the solution with the larger function value.  
Though we use  \algname{Non-convex Frank-Wolfe} 
as the subroutine here, it is worth noting that any algorithm that is 
guaranteed to find an approximately stationary point can
be plugged into \cref{lg_fw} as the subroutine. 
We give an improved approximation bound  by
considering more properties of DR-submodular functions.  
Borrowing the results from \citet{lacoste2016convergence} for the 
\algname{non-convex Frank-Wolfe} subroutine, we get  the following, 
\begin{theorem}\label{rate_local_fw}
	The output of \cref{lg_fw} satisfies,
	\begin{align}\label{eq_local_rates} 
	&\max \{f(\x), f(\z)\}  \geq 
	 \frac{\mu}{8}\left(\|\x -\x^*\|^2 + \|\z - \z^*\|^2\right )\\\notag 
	& \qquad \qquad  + \frac{1}{4}\left[f(\x^*)  - \min \left\{\frac{\max \{2h_1, C_f(\P)\}}{\sqrt{K_1+1}} , \epsilon_1\right\}   - \min\left\{\frac{\max \{2h_2, C_f(\Q)\}}{\sqrt{K_2+1}} , \epsilon_2\right\} \right], 
	\end{align}  
	where $h_1 := \max_{\x\in\P}f(\x) - f(\x^\pare{0})$, $h_2 := \max_{\z\in\Q}f(\z) - f(\z^\pare{0})$ are the initial suboptimalities,  
	$C_f(\P) : = \sup_{\x, \v\in \P, \gamma\in [0,  1],  \y = \x + \gamma (\v - \x )}\frac{2}{\gamma^2}(f(\y) - f(\x) - {(\y - \x)^\trans}{\nabla f(\x)})$ is  the curvature of  $f$ w.r.t.   $\P$, and $\z^*= \x\vee \x^* -\x$.
\end{theorem}
\cref{rate_local_fw} indicates that \cref{lg_fw}
has a $1/4$ approximation guarantee and $1/\sqrt{k}$ rate.
However, it has good empirical performance as demonstrated by the experiments in \cref{sec_exp}.
Informally, this can be partially explained by the  term 
 $\frac{\mu}{8}\left(\|\x -\x^*\|^2 + \|\z - \z^*\|^2\right )$ in 
\labelcref{eq_local_rates}: if $\x$
is away from $\x^*$, this term will augment the bound; if $\x$ is close to $\x^*$,
by the smoothness of $f$, it should be close 
to optimal.





\subsection{The Non-monotone \algname{Frank-Wolfe} Variant}\label{sec_fw_variant}

\IncMargin{1em}
\begin{algorithm}[htbp]
	\caption{Non-monotone \algname{Frank-Wolfe} variant for 
		{DR}-submodular
		maximization}\label{fw-non-monotone}
	\KwIn{$\max_{\x \in \P} f(\x)$,
		prespecified step size $\gamma \in (0, 1]$}
	{$\x^\pare{0} \leftarrow 0$, $t^\pare{0}\leftarrow 0$, $k\leftarrow 0$\tcp*{$k:$ iteration index, $t^\pare{k}:$ cumulative step size}}
	\While{$t^\pare{k} <  1$}{
		{$\v^\pare{k} \leftarrow   \argmax_{\v\in\P, \textcolor{blue}{\v\leq {\bar \u}-\x^\pare{k}}} \dtp{\v}{ \nabla f(\x^\pare{k})}$\tcp*{\emph{shrunken LMO}} \label{new_lmo}}
		{use uniform step size $\gamma_k = \gamma$;  set $\gamma_k \leftarrow \min\{\gamma_k, 1-t^\pare{k} \}$\;}
		{$\x^\pare{k+1}\leftarrow \x^\pare{k} + \gamma_k \v^\pare{k}$, $t^\pare{k+1} \leftarrow t^\pare{k} + \gamma_k$,  $k\leftarrow k+1$\;}
	}
	\KwOut{$\x^\pare{K}$ \tcp*{assuming there are $K$ iterations in total}}
\end{algorithm}
\DecMargin{1em}


\cref{fw-non-monotone} summarizes the non-monotone \algname{Frank-Wolfe} variant, which  is inspired by the  unified continuous
greedy algorithm in \citet{feldman2011unified} for
maximizing the multilinear extension of a submodular
set function.
It initializes the solution $\x^\pare{0}$ to be 0, 
and maintains $t^\pare{k}$ as the cumulative
step size. At iteration $k$,
it  maximizes the linearization 
of $f$ over a ``shrunken'' constraint
set: $\{\v| \v\in \P, \v\leq \bar \u-\x^\pare{k}\}$, which
is different from the classical LMO 
of Frank-Wolfe-style algorithms (hence we refer to it as the ``shrunken LMO''). Then it employs an update step in the direction $\v^\pare{k}$ chosen by the LMO 
with a   uniform 
step size $\gamma_k = \gamma$.
The cumulative
step size $t^\pare{k}$ is used to ensure that 
the overall step sizes sum to one, thus the output
solution $\x^\pare{K}$ is a convex combination
of the  LMO outputs, hence also lies
in $\P$.

The shrunken LMO (Step \labelcref{new_lmo}) is  
the key difference compared to  the monotone  \algname{Frank-Wolfe} variant in  \citet{bian2017guaranteed}.  The extra  constraint $\v\leq {\bar \u}-\x^\pare{k}$ is added to prevent too aggressive growth of the 
solution, since in the non-monotone setting such aggressive growth  may hurt the overall performance. 
The next theorem states the guarantees of  \cref{fw-non-monotone}. 
\begin{theorem}[]\label{thm-e}
	Consider  \cref{fw-non-monotone} with uniform step size $\gamma$.   For $k = 1,..., K$ it holds that,
	\begin{flalign}
	f(\x^\pare{k}) \geq t^\pare{k} e^{-t^\pare{k}}f(\x^*) - \frac{L D^2}{2}k\gamma^2 - O(\gamma^2)f(\x^*). 
	\end{flalign}
\end{theorem}
By observing  that $t^\pare{K} = 1$ and applying \cref{thm-e}, we get the following \namecref{coro_e}: 
\begin{corollary}[]\label{coro_e}
	The output of \cref{fw-non-monotone}  satisfies
	$f(\x^\pare{K}) \geq  e^{-1}f(\x^*) - \frac{L D^2}{2K} - O\left(\frac{1}{K^2}\right)f(\x^*)$. 
\vspace{-0.2cm}
\end{corollary}

\cref{coro_e} shows that \cref{fw-non-monotone} 
enjoys a sublinear convergence rate towards
some point $\x^\pare{K}$ inside $\P$, with a
$1/e$ approximation guarantee. 

\textbf{Proof sketch of \cref{thm-e}: }
The proof is by induction. To prepare the building blocks, we first of all  show that  the growth of $\x^\pare{k}$ is indeed bounded,
\begin{restatable}[]{lemma}{restalemmatwo}
\label{prop_non_fw}
	Assume $\x^\pare{0} = 0$. For $k=0,..., K-1$, it holds  $x_i^\pare{k}\leq \bar u_i[1-(1-\gamma)^{t^\pare{k}/\gamma}], \forall i\in [n]$.
\end{restatable}
Then the following \namecref{lem_nonmonotone_fw} provides a lower bound, which gets  the global optimum involved, 
\begin{restatable}[Generalized from Lemma 7 in  \citet{chekuri2015multiplicative}]{lemma}{restalemmathree}
	\label{lem_nonmonotone_fw}
Given  $\bmtheta\in (0, \bar \u]$, let $\lambda' = \min_{i\in [n]} \frac{\bar u_i}{\theta_i}$. Then for all  $\x\in [0, \bmtheta]$, it holds  $f(\x\vee \x^*) \geq (1-\frac{1}{\lambda'})f(\x^*)$. 
\end{restatable}
Then the key ingredient for induction  is the relation between  $f(\x^{\pare{k+1}})$
and $f(\x^{\pare{k}})$ indicated by:  
	\begin{restatable}{claim}{restaclaimthree}
	\label{claim3_1}
	For $k = 0,...,K-1$ it holds
	$f(\x^{\pare{k+1}}) \geq (1-\gamma)  f(\x^{\pare{k}})   + \gamma(1-\gamma)^{t^\pare{k}/\gamma} f(\x^*) -\frac{L D^2}{2}\gamma^2$,
\end{restatable}			
		which is derived by a combination of
		the quadratic lower bound in \labelcref{eq_quad_lower_bound},
\cref{prop_non_fw} and  \cref{lem_nonmonotone_fw}.

\textbf{Remarks on the two algorithms.}
Notice that  though the \algname{two-phase}  algorithm has a worse 
guarantee than the non-monotone \algname{Frank-Wolfe} variant, it is still of  interest: i) It allows 
 flexibility in using a wide range of existing solvers for finding an (approximately) stationary point. ii) The guarantees that we present  rely on a worst-case analysis. The empirical performance of the \algname{two-phase} algorithm is often comparable or 
better than that of  the \algname{Frank-Wolfe} variant. This  suggests to 
explore more  properties in concrete problems that may favor the \algname{two-phase}
algorithm,  which we leave  for  future work.

%

%% file: exp.tex
\vspace{-0.2cm}
\section{Experimental Results}
\label{sec_exp}

We test the performance 
of the analyzed algorithms, 
while considering   the following 
baselines: 1) \algname{quadprogIP}  \citep{xia2015globally}, which
is a global solver for non-convex quadratic programming; 2) Projected
gradient ascent (\algname{ProjGrad}) with diminishing step sizes ($\frac{1}{k+1}$, $k$ starts from 0).
 We run all the algorithms for 100 iterations. 
For the subroutine (\cref{nonconvex_fw}) of \algname{two-phase Frank-Wolfe}, we set
$\epsilon_1 = \epsilon_2 = 10^{-6}, K_1 = K_2 = 100$.
All the synthetic results   are  the 
average of 20 repeated experiments. 
All experiments were implemented using MATLAB.  Source code can be found at: 
\url{https://github.com/bianan/non-monotone-dr-submodular}.

\subsection{DR-submodular Quadratic Programming}

As a state-of-the-art global solver,  \algname{quadprogIP}\footnote{We used  the open source code provided by \citet{xia2015globally}, and the IBM CPLEX optimization studio {\url{https://www.ibm.com/jm-en/marketplace/ibm-ilog-cplex}}
as the subroutine.} \citep{xia2015globally}  
can find the global optimum (possibly  in  exponential time), which were used  to calculate  the approximation ratios.
Our problem instances are  synthetic DR-submodular quadratic objectives with down-closed  polytope
constraints, i.e., $f(\x) = \frac{1}{2}\x^\trans \bmH \x + \h^\trans \x +c$ and $\P = \{\x\in \R_+^n \ |\  \bmA \x \leq \b, \x \leq \bar \u, \bmA\in \R_{++}^{m\times n}, \b\in \R_+^m \}$. 
Both objective and constraints were randomly generated, in the 
following two manners: 

\textbf{1) Uniform distribution. }
$\bmH\in \R^{n\times n}$ is a symmetric matrix 
with uniformly distributed  entries in $[-1, 0]$; $\bmA\in \R^{m\times n}$ has uniformly distributed entries in $[\nu, \nu +1]$, where $\nu = 0.01$
is a small positive constant in order to make entries of $\bmA$
strictly positive.  


\setkeys{Gin}{width=0.33\textwidth}
 \begin{figure}[htbp]
   \center 
  \includegraphics[width=0.56\textwidth]{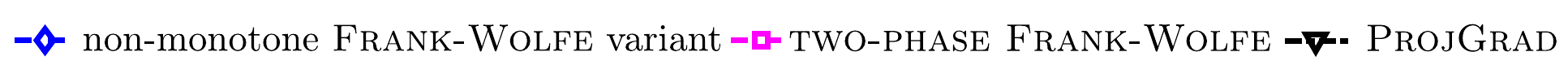}\\
  \vspace{-0.4cm}
 \subfloat[$m={\floor {0.5n}}$ \label{fig_quad_sub1}]{
 \includegraphics[]{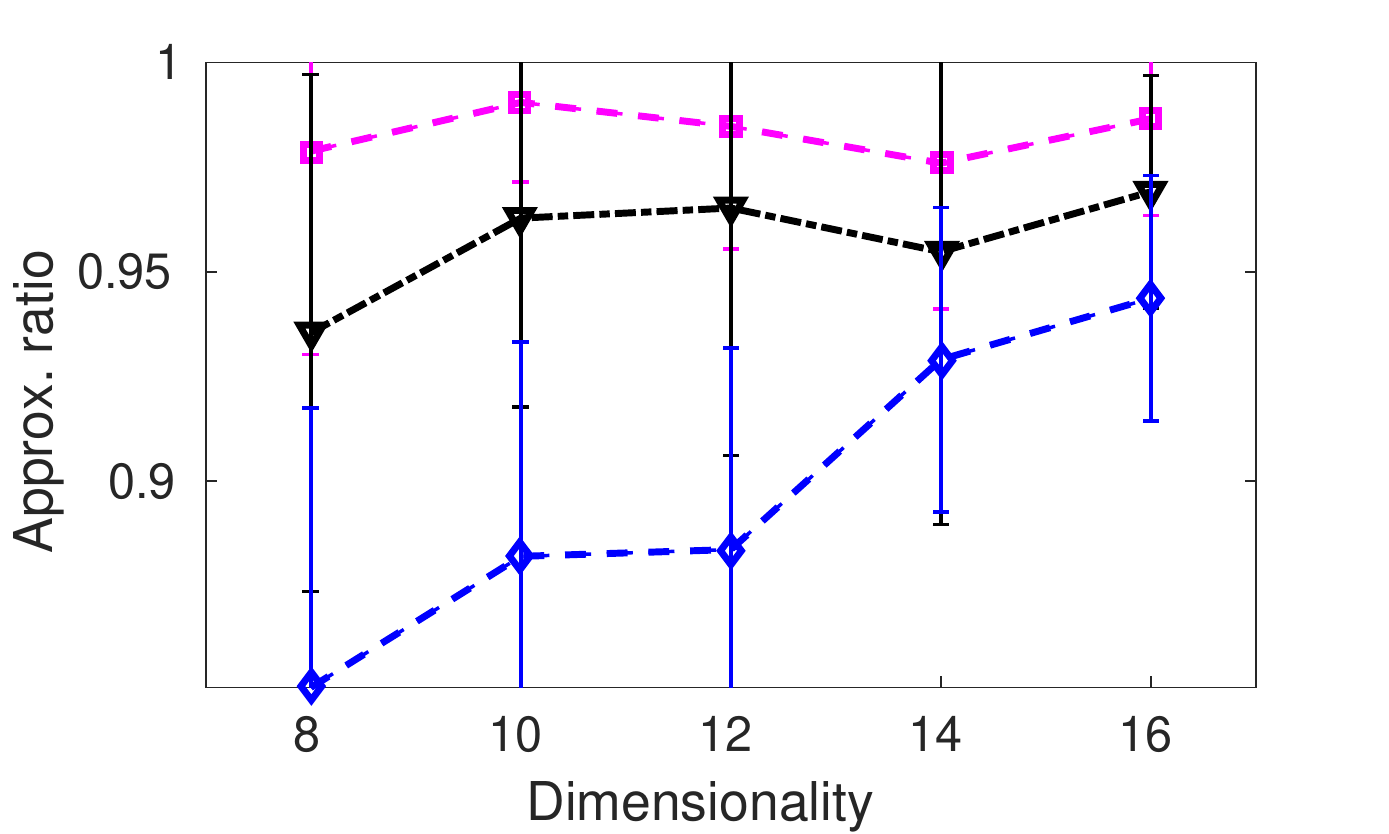}}
 \hspace{-0.4cm}
 \subfloat[$m=n$ \label{fig_quad_sub2}]{
 \includegraphics[]{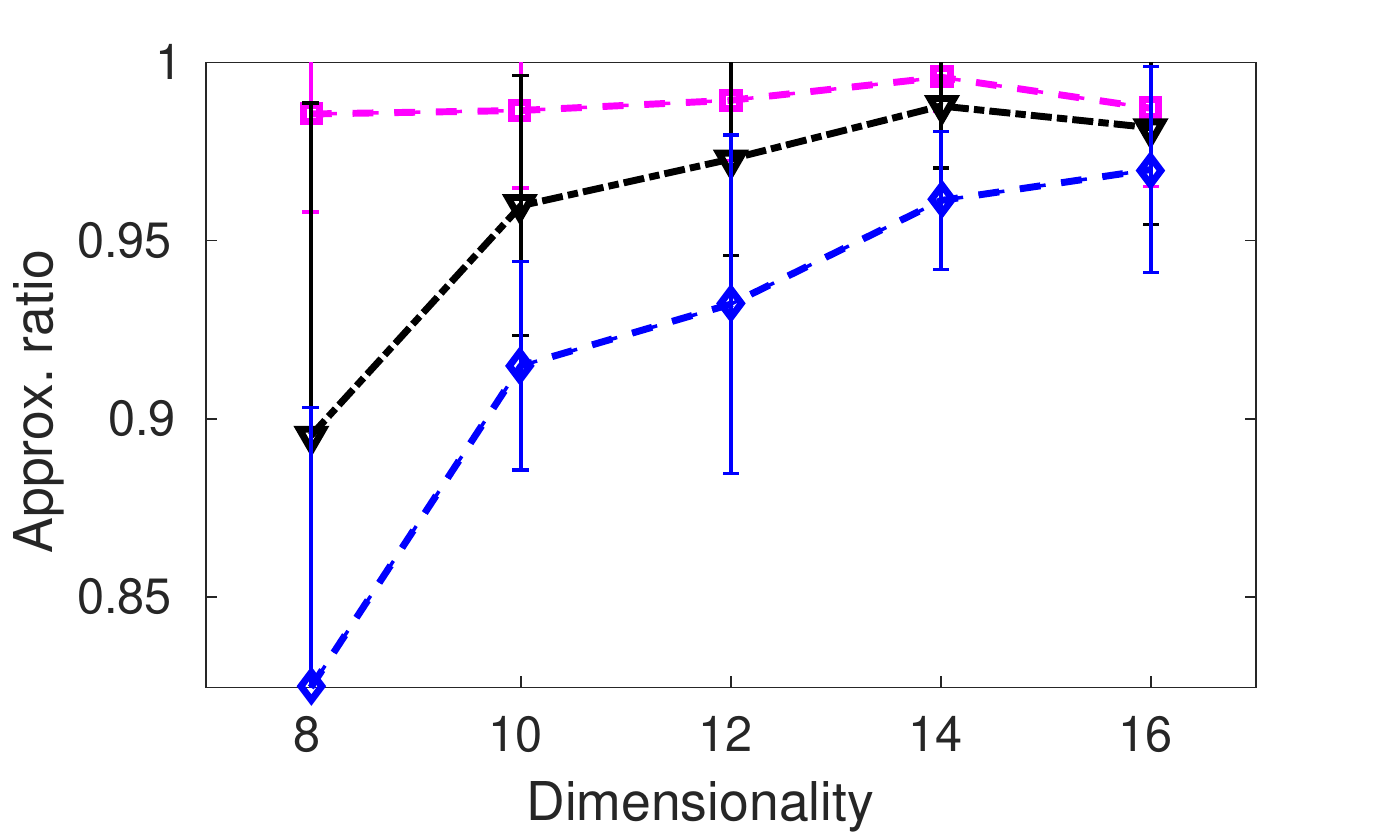}}
  \hspace{-0.4cm}
  \subfloat[$m=\floor {1.5n}$ \label{fig_quad_sub3}]{
  \includegraphics[]{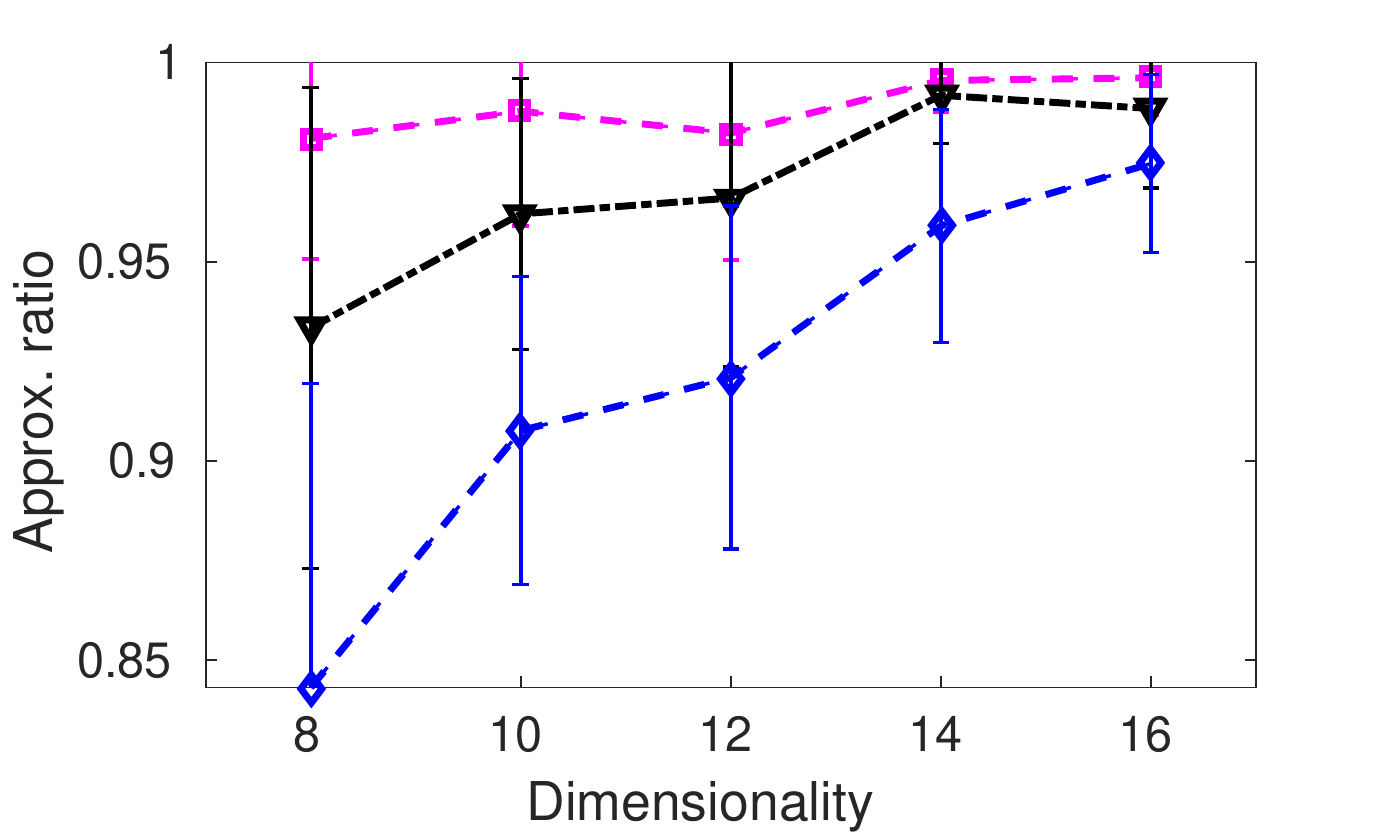}}
 \caption{Results on   DR-submodular  quadratic instances with uniform distribution.}
 \label{fig_quad}
\end{figure}

\textbf{2) Exponential distribution. }  The entries of $-\bmH$
and $\bmA$ were sampled from exponential distributions $\text{Exp}(\lambda)$ (For a random variable $y\geq 0$, its probability density function is $\lambda e^{-\lambda y}$, and for $y<0$, its  density is $0$).    Specifically, each entry of $-\bmH$
was sampled from $\text{Exp}(1)$, then the matrix $-\bmH$
was made to be  symmetric. Each entry of $\bmA$ was
sampled from $\text{Exp}(0.25) + \nu$, where $\nu = 0.01$
is a small positive constant.

In both the above two cases, we set   $\b = \mathbf{1}^m$, and    $\bar \u$ to be the tightest upper bound of $\P$ by  $\bar u_j = \min_{i\in [m] }\frac{b_i}{A_{ij}}, \forall j\in [n]$. 
In order to make $f$  non-monotone, 
we set $\h = -0.2*\bmH^\trans \bar \u$. 
To make sure that $f$ is non-negative, we first of all solve the 
problem $\min_{\x\in \P} \frac{1}{2}\x^\trans \bmH \x + \h^\trans \x$ using \algname{quadprogIP}, let the solution to be $\hat\x$, then  set 
$c= -f(\hat\x) + 0.1*|f(\hat\x)|$. 

 \setkeys{Gin}{width=0.33\textwidth}
 \begin{figure}[htbp]
   \center 
  \includegraphics[width=0.56\textwidth]{legend_h.pdf}\\
  \vspace{-0.4cm}
 \subfloat[$m= {\floor {0.5n}}$ \label{fig_quad_exp_sub1}]{
 \includegraphics[]{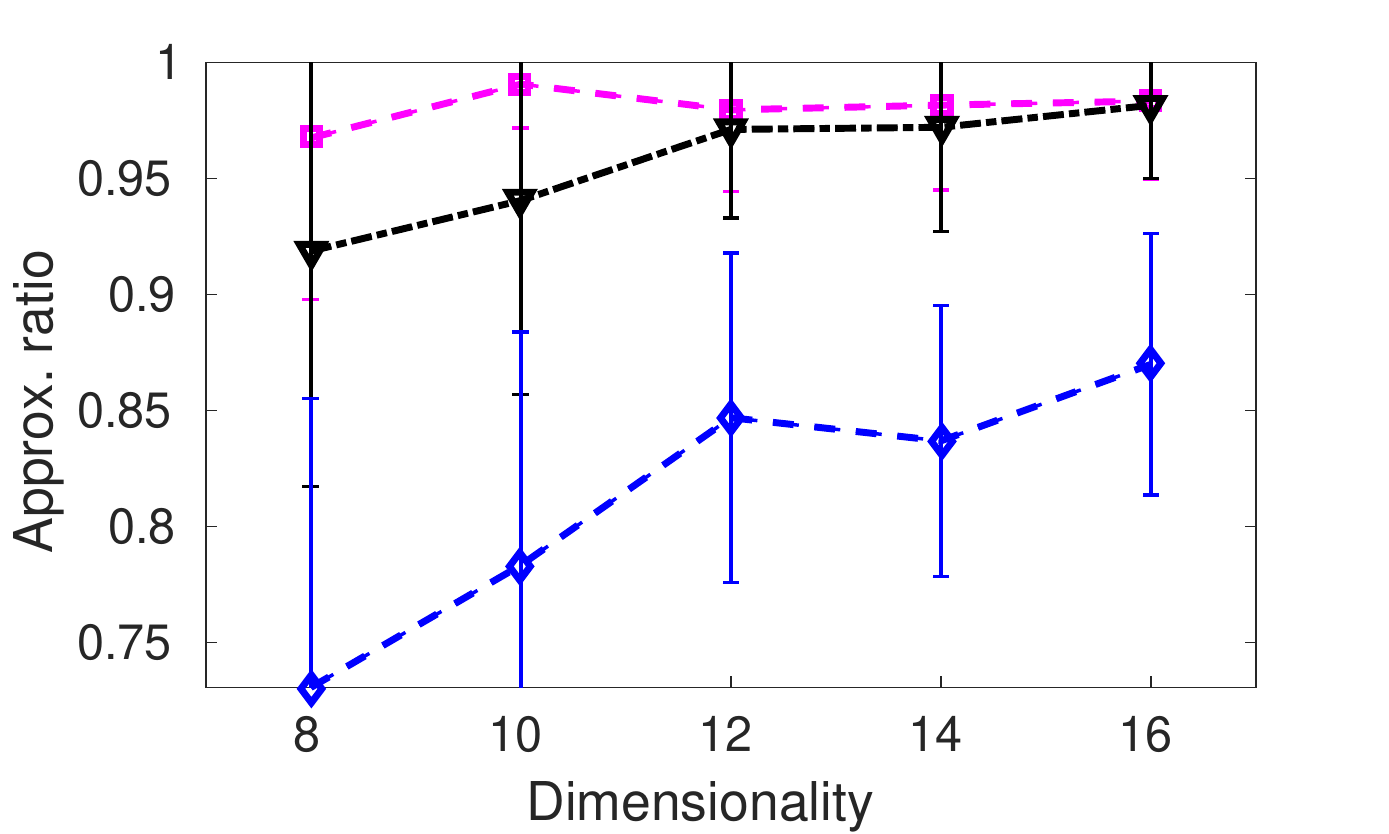}}
 \hspace{-0.4cm}
 \subfloat[$m=n$ \label{fig_quad_exp_sub2}]{
 \includegraphics[]{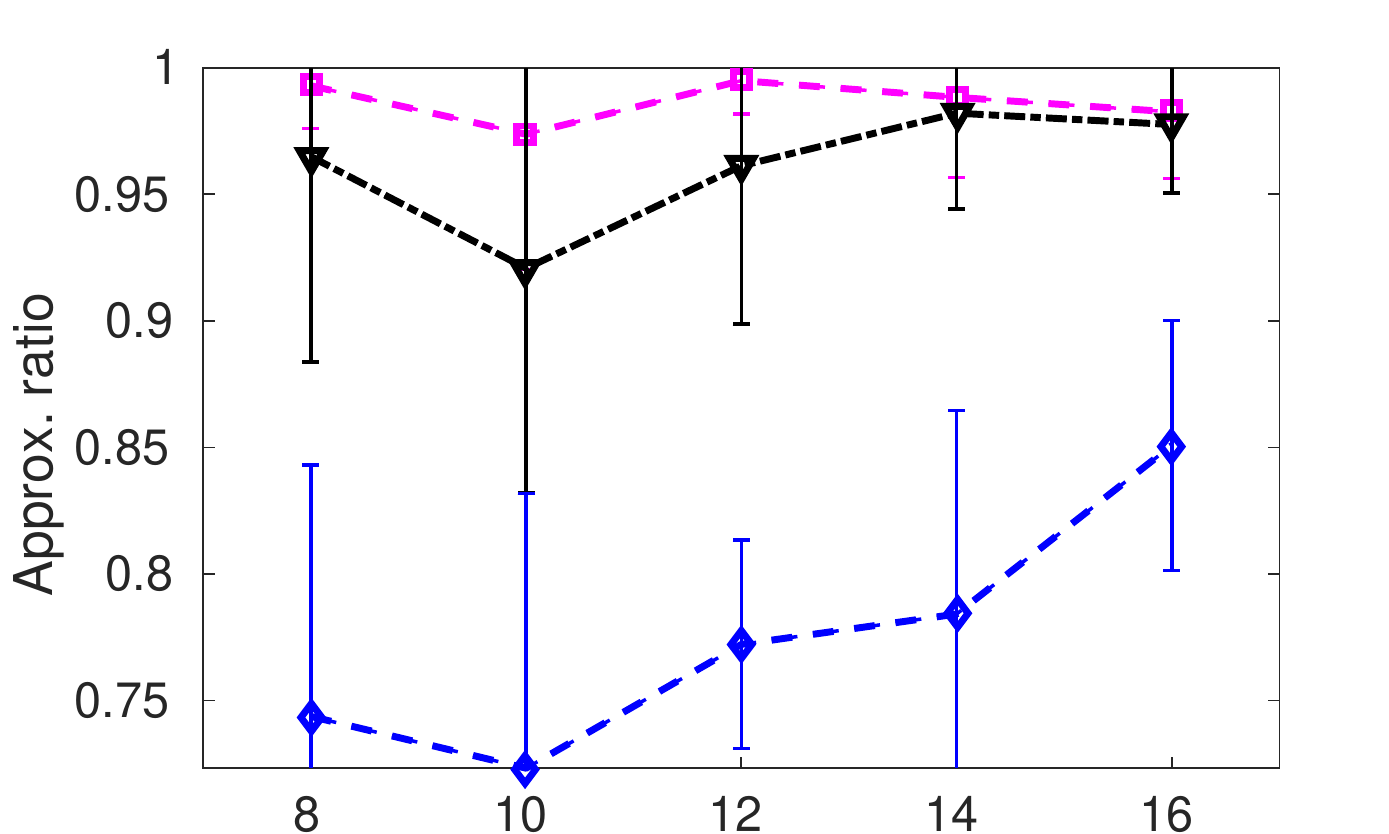}}
  \hspace{-0.4cm}
  \subfloat[$m=\floor {1.5n}$ \label{fig_quad_exp_sub3}]{
  \includegraphics[]{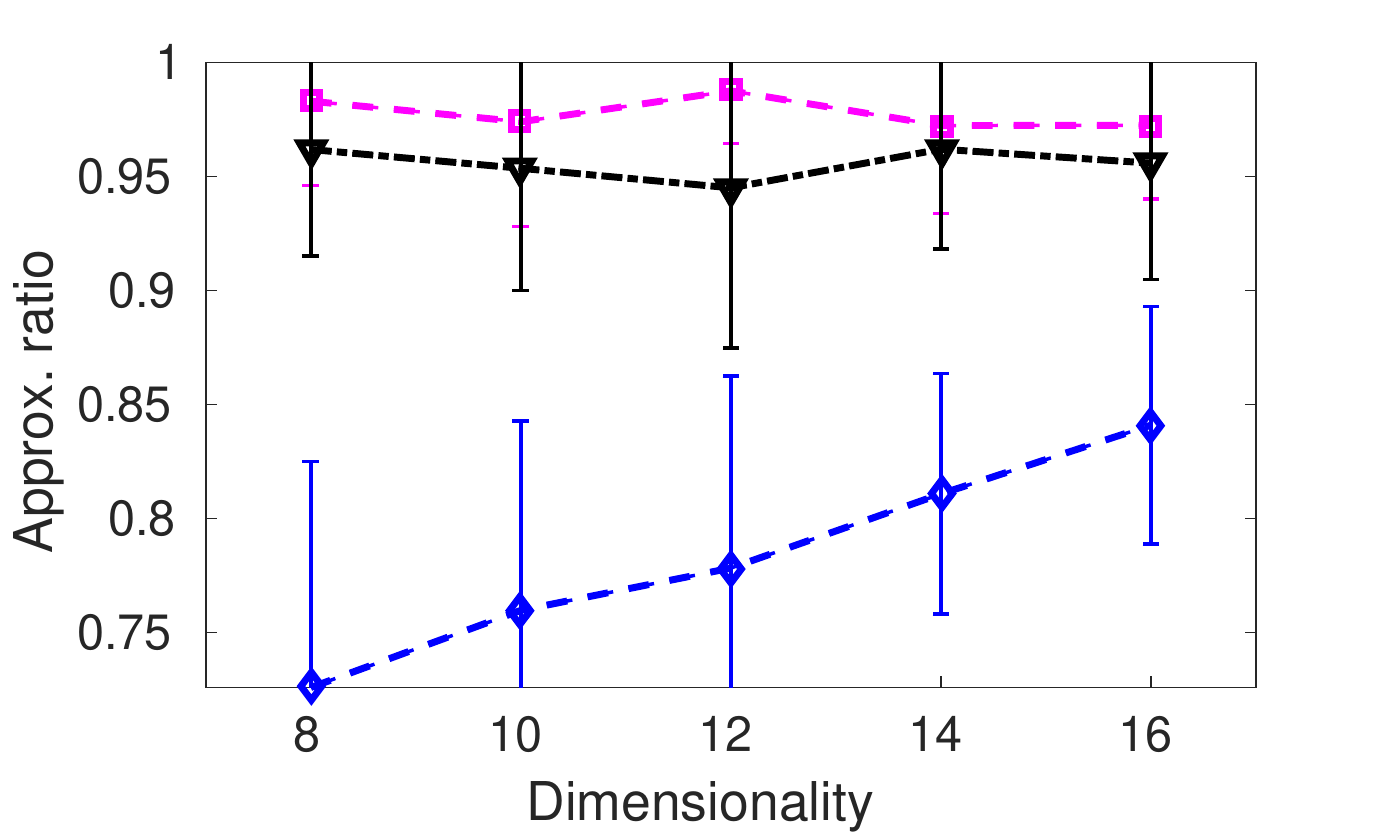}}
 \caption{Results on   quadratic instances with exponential  distribution.}
 \label{fig_quad_exp}
 \vspace{-0.44cm}
\end{figure}

The approximation ratios w.r.t.  dimensionalities ($n$) are plotted in \cref{fig_quad,fig_quad_exp}, for the two manners of data generation. We set the number of constraints to be $m=\floor {0.5n}$,  
$m=n$ and $m=\floor {1.5n}$ in \cref{fig_quad_sub1,fig_quad_sub2,fig_quad_sub3} (and \cref{fig_quad_exp_sub1,fig_quad_exp_sub2,fig_quad_exp_sub3}), respectively.

One can see that 
  \algname{two-phase Frank-Wolfe} 
usually performs  the best, \algname{ProjGrad} follows, and  non-monotone 
 \algname{Frank-Wolfe} variant is the last. 
 The good performance of \algname{two-phase Frank-Wolfe}  can be partially  explained by the strong DR-submodularity
 of quadratic functions according to \cref{rate_local_fw}.
Performance of the two analyzed algorithms is consistent with the  theoretical
bounds: the approximation ratios of  \algname{Frank-Wolfe} variant
are always much higher than $1/e$.

\subsection{Maximizing  Softmax Extensions}

 With some derivation, one can see the derivative of the softmax extension  in \labelcref{eq_softmax} is:
 $\nabla_i f(\x) = \tr{ ({\diag(\x)(\bmL-\bmI) +\bmI })^{-1}(\bmL - \bmI)_i}, \forall i\in [n]$, 
 where $(\bmL - \bmI)_i$ denotes the matrix obtained by zeroing  all entries except
 the $i^\text{th}$ row of $(\bmL - \bmI)$. Let $\bmC:= ({\diag(\x)(\bmL-\bmI) +\bmI })^{-1}, \bmD:=(\bmL - \bmI)$, one can see that $\nabla_i f(\x) = \bmD_{i\cdot}^\trans \bmC_{\cdot i}$, which gives an efficient way to calculate the gradient $\nabla f(\x)$.


  \setkeys{Gin}{width=0.33\textwidth}
         \begin{figure}[htbp]
       \center 
      \includegraphics[width=0.58\textwidth]{legend_h.pdf}\\
      \vspace{-0.4cm}
              \subfloat[$m={\floor {0.5n}}$ \label{fig_softmax_syn1}]{
              \includegraphics[]{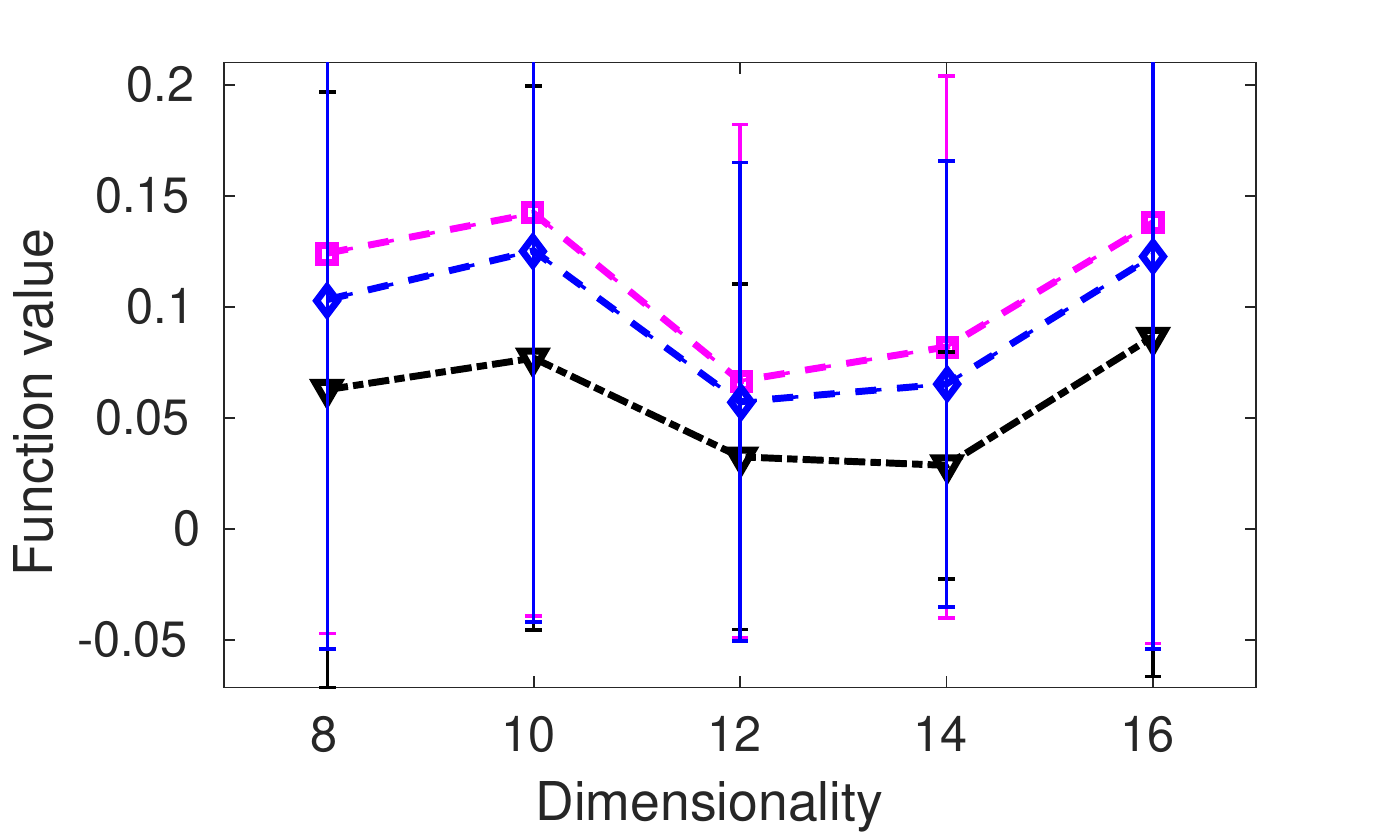}}
        \subfloat[$m=n$ \label{fig_softmax_syn2}]{
        \includegraphics[]{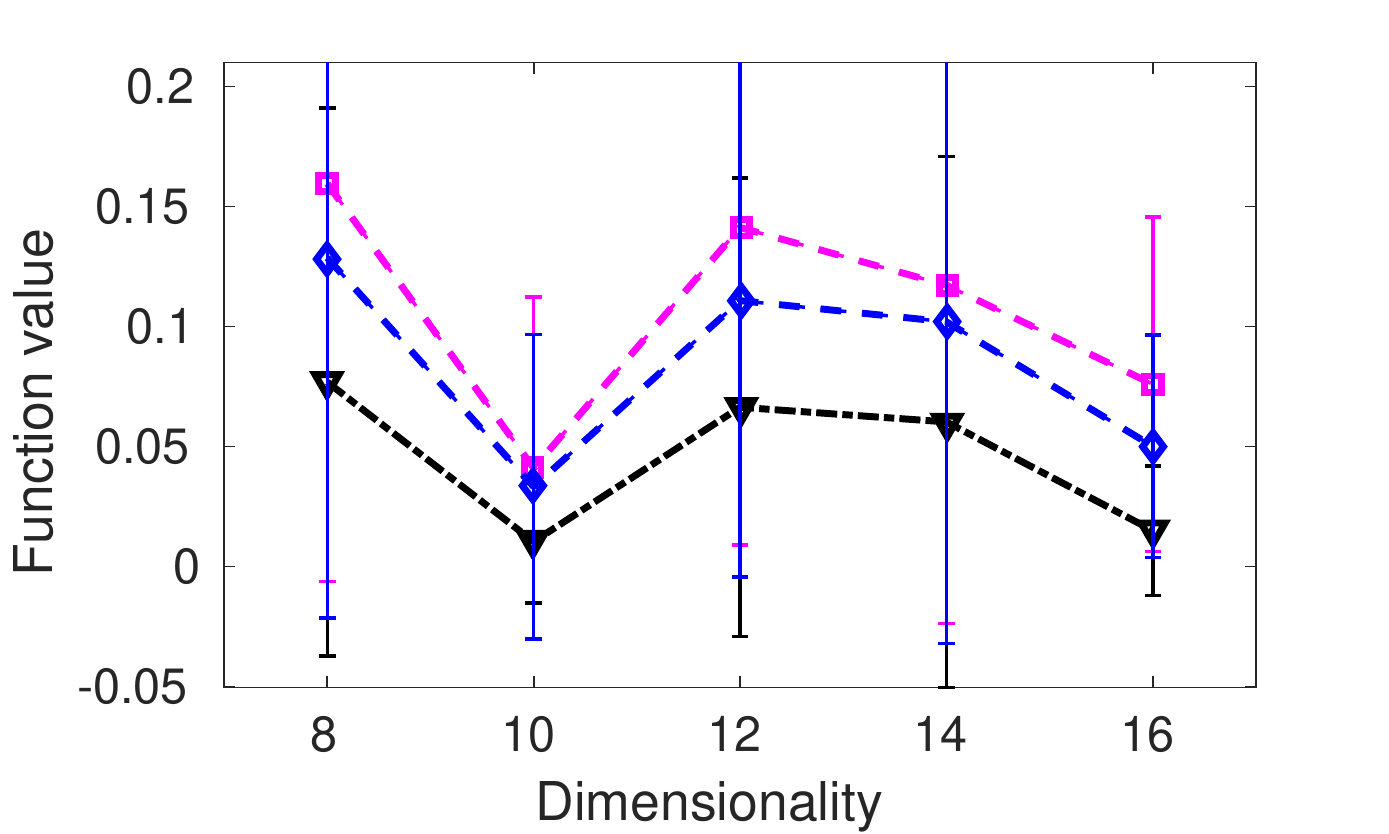}}
      \subfloat[$m= \floor{1.5n}$ \label{fig_softmax_syn3}]{
      \includegraphics[]{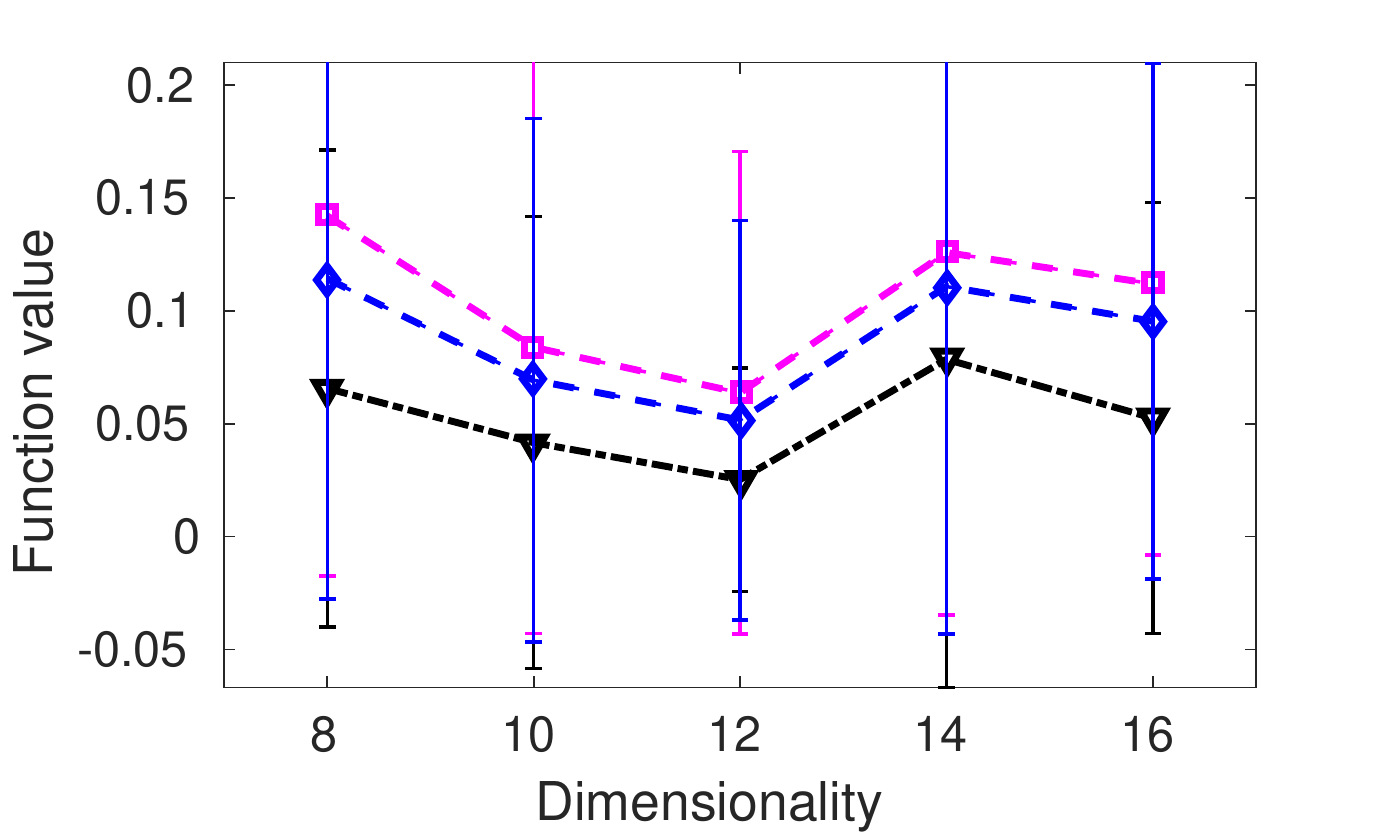}}
    \caption{Results  on  softmax instances with polytope constraints generated from uniform distribution.}
         \label{fig_softmax_syn}
    \end{figure}   
 
 \textbf{Results on synthetic data.}
We generate the softmax objectives (see \labelcref{eq_softmax}) in the following way:  first 
 generate the $n$ eigenvalues $\d\in \R_+^n$, each 
 randomly distributed in $[0, 1.5]$, and set  
  $\bmD = \diag(\d)$. After   
 generating a random unitary matrix $\bmU$, we   set $\bmL = \bmU \bmD\bmU^\trans$.   One can verify that $\bmL$ is positive semidefinite and has
 eigenvalues as the entries of $\d$.

 We generate the down-closed polytope constraints in the same form  
 and same way as that for  DR-submodular quadratic functions,
 except for setting $\b = 2*\mathbf{1}^m$.
Function values returned by different solvers w.r.t. $n$ are shown in \cref{fig_softmax_syn}, for which the random polytope
constraints were generated with uniform distribution (results for which the random polytope constraints were generated with 
exponential distribution are deferred to \cref{app_add_exp}).
The number of constraints was set to be $m={\floor {0.5n}}$,  $m=n$ and 
$m=\floor {1.5n}$ in  \cref{fig_softmax_syn1,fig_softmax_syn2,fig_softmax_syn3}, respectively. 
One can observe that  \algname{two-phase Frank-Wolfe}
still has the best performance, the non-monotone
\algname{Frank-Wolfe} variant follows, and   \algname{ProjGrad} has the worst performance.

\textbf{Real-world results on   matched summarization.} 
The task of ``matched summarization'' is to select a set of document \emph{pairs} out of a corpus of documents, 
  such that the two documents within a pair are similar, and the overall set of 
  pairs is as diverse as possible. 
  The motivation for this task is very practical: it  could be, for example,  to compare the opinions of various
  politicians on a range of representative topics. 

In our experiments, we used a similar setting to the one in  \citet{gillenwater2012near}.
We experimented on the 2012 US Republican  debates data, which consists
of 8  candidates: Bachman, Gingrich, Huntsman, Paul, Perry, Romney
and Santorum. Each task involves one pair of candidates, so in total
there are $28=7*8/2$ tasks. \cref{fig_softmax1} plots the averaged function
values returned by the three solvers over 28 tasks, w.r.t. different values of
a hyperparameter reflecting the matching quality (details see   \citet{gillenwater2012near}). 
\setkeys{Gin}{width=0.33\textwidth}
\begin{wrapfigure}[]{l}{0.66\textwidth}
	\centering
	\includegraphics[width=0.56\textwidth]{legend_h.pdf}\\
	\vspace{-0.5cm}
	\subfloat[Average on 28 tasks \label{fig_softmax1}]{
		\includegraphics[]{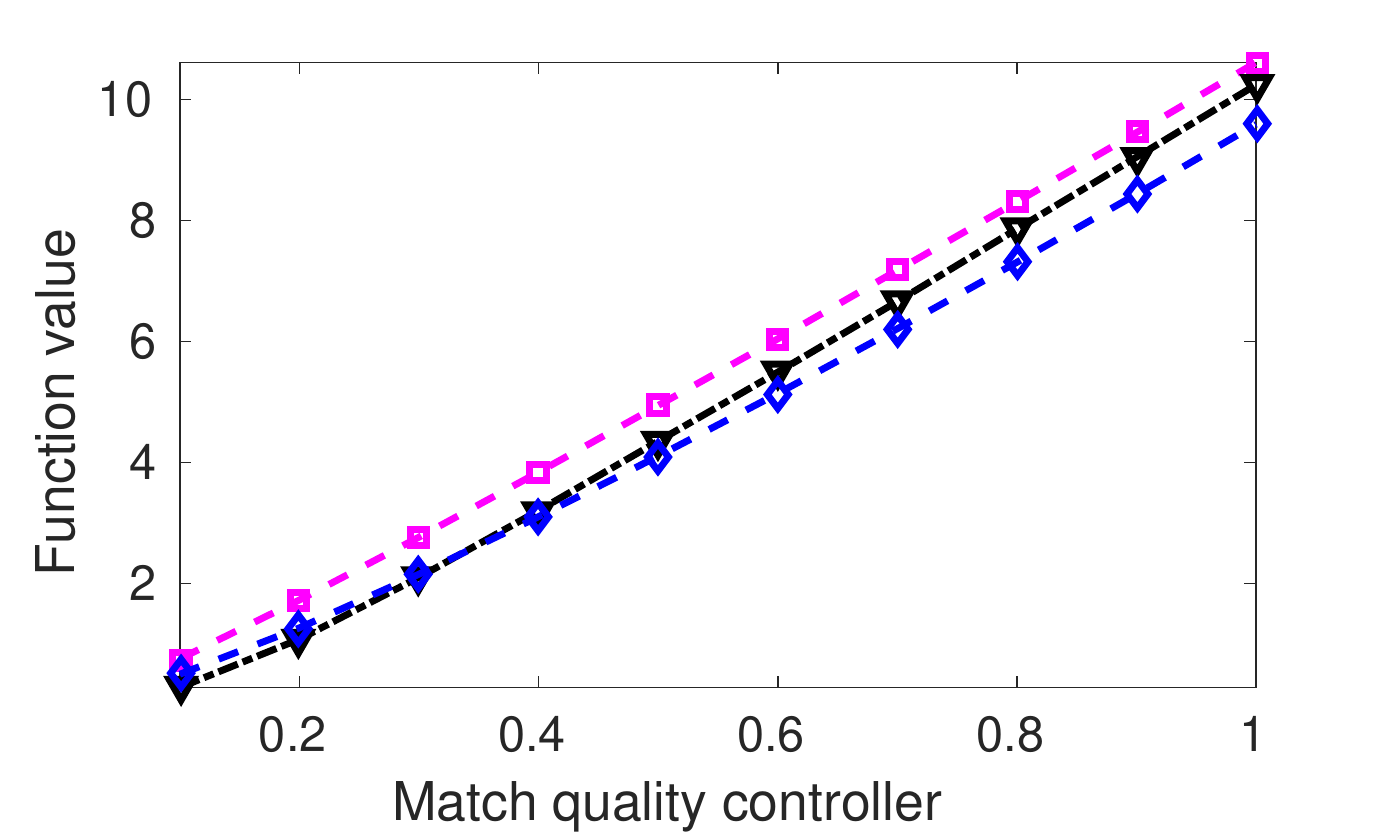}}
	\hspace{-.4cm}
	\subfloat[Objectives w.r.t. iterations \label{fig_softmax3}]{
		\includegraphics[]{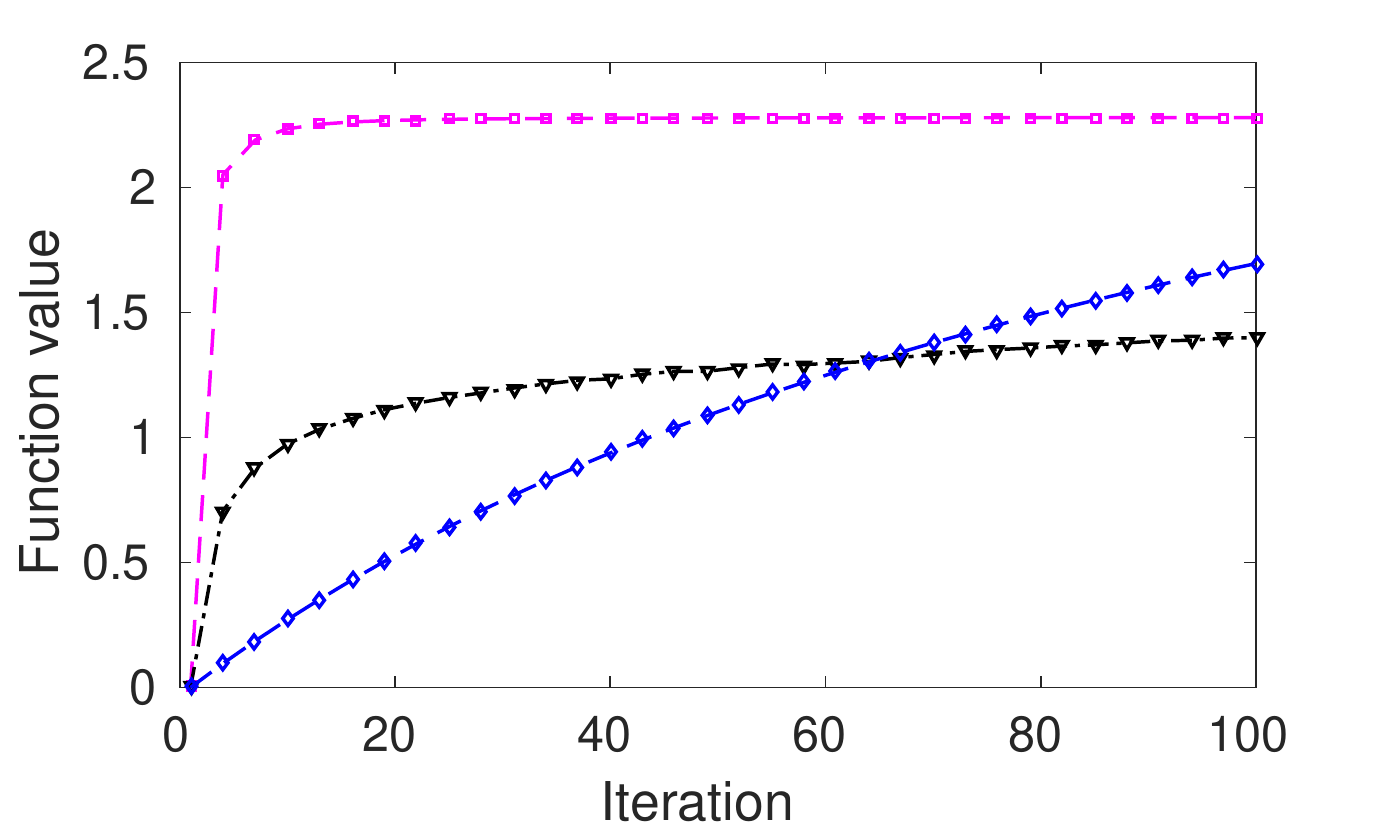}}
	\caption{Results  on  2012 US Republican  debates data.}
	\label{fig_softmax}
	 \vspace{-.2cm}
\end{wrapfigure} 
\cref{fig_softmax3}
traces the objectives w.r.t. iterations for a specific candidate pair (Bachman, Romney). 
For \algname{two-phase Frank-Wolfe}, the objectives of the selected phase
were  plotted.
One can see that  \algname{ two-phase Frank-Wolfe} also 
achieves the best performance, while 
the performance of non-monotone \algname{Frank-Wolfe}
variant and \algname{ProjGrad} is comparable. 

%
%

%
%
%
%

%% file: related.tex
\section{Related Work}\label{sec_related}

Submodular optimization and, more broadly, non-convex
optimization are extensively studied in the literature,  which renders it
very difficult comprehensively surveying all previous work.
Here  we only  briefly summarize some of the most related papers.

\textbf{Submodular optimization over integer-lattice and continuous
  domains.}  Many results from submodular set function optimization
have been generalized to the integer-lattice case
\citep{soma2014optimal,DBLP:conf/nips/SomaY15,ene2016reduction,khodabakhsh2016maximizing}.
Of particular interest is the reduction \citep{ene2016reduction} from
an integer-lattice DR-submodular maximization problem to a submodular
set function maximization problem. %
Submodular optimization over continuous domains has attracted
considerable attention recently
\citep{bach2015submodular,bian2017guaranteed,staibrobust}.
%
Two  classes of functions that are covered by continuous
submodularity are the Lovasz extensions \citep{lovasz1983submodular}
and multilinear extensions \citep{calinescu2007maximizing} of
submodular set functions.  Particularly, multilinear extensions of
submodular set functions are also continuous DR-submodular
\citep{bach2015submodular}, but with the special property that they are
coordinate-wise linear.  Combined with the rounding technique of
contention resolution \citep{chekuri2014submodular}, maximizing
multilinear extensions
\citep{DBLP:conf/stoc/Vondrak08,gharan2011submodular,feldman2011unified,chekuri2015multiplicative,ene2016constrained}
has become the state-of-the-art method for submodular set function
maximization.
Some of the techniques in maximizing multilinear extensions
\citep{feldman2011unified,chekuri2014submodular,chekuri2015multiplicative} have inspired
this work.  However, we are the first to explore the rich properties
and devise algorithms for the general constrained DR-submodular
maximization problem over continuous domains.

\textbf{Non-convex optimization.}  Non-convex optimization receives a
surge of attention in the past years. One active research topic
 is to reach a stationary point for unconstrained optimization
\citep{sra2012scalable,reddi2016fast,allen2016variance} or constrained
optimization \citep{ghadimi2016mini,lacoste2016convergence}.  However,
without proper assumptions, a stationary point may not lead to any
global approximation guarantee.  The local-global relation (in
\cref{local_global}) provides a strong relation between (approximately) stationary
points and global optimum, thus making it flexible to incorporate
progress in this area.

%



%
%
%

%% file: disc.tex


\section{Conclusion}

We have studied the problem of constrained non-monotone DR-submodular
continuous maximization.  We explored the structural properties of
such problems, and established a local-global relation.  Based on these properties, we presented a \algname{two-phase} algorithm
with a $1/4$ approximation guarantee, and a non-monotone
\algname{Frank-Wolfe} variant with a $1/e$ approximation guarantee.
We further generalized submodular continuous function over  conic
lattices, which enabled us to model a larger class of applications.
Lastly, our theoretical findings were verified by synthetic and
real-world experiments.

\paragraph{Acknowledgement.}
This research was partially supported by  ERC StG 307036, by the  
Max Planck ETH Center for Learning Systems, and by the  ETH Z\"urich Postdoctoral Fellowship program.

%% file: lattice.tex
\section{Continuous Submodular  Functions  on  Conic Lattices and the Reduction}\label{sec_lattice}

Motivated by the objectives  that can not be modeled
by   continuous  submodular functions, we consider the more general submodular continuous functions over  lattices  induced by conic inequalities. Furthermore  we provide a reduction to the original (DR-)submodular optimization  problem. 

%

\subsection{Definitions and Properties}
Let us look
at the proper cone that will be used to define a conic inequality firstly.
	 A cone $\cone\subseteq \R^n$ is a \emph{proper cone} if it  is convex, closed, solid (having nonempty interior) and  pointed (contains no line, i.e., $\x\in \cone, -\x\in \cone$ implies  $\x=0$).
A proper cone $\cone$ can be used to define a conic inequality (a.k.a.  generalized inequality \citep[Chapter 2.4]{boyd2004convex}): $\a\preceq_{\cone} \b$ 
 iff $\b-\a \in \cone$, which also defines a partial ordering 
 since the binary relation $\lleq_{\cone}$ 
 is reflexive, antisymmetric and transitive. 
 Then it is easy to see that $(\X, \lleq_\cone)$ is a partial ordered set (poset).

 If two elements $\a, \b \in \X$ have a least upper bound (greatest 
 lower bound), it is denoted as the ``join'':  $\a\vee \b$ (the ``meet'': $\a\wedge \b$).  
A lattice is a poset that contains the join and meet of each pair of
its elements \citep{garg2015introduction}. 
%
A ``lattice cone''  \citep{fuchssteiner2011convex} is the proper cone that  can be used to define a lattice. 
%
Note that not all conic inequalities can be used to define a lattice. {For example, the positive semidefine 
	cone $\cone_{\text{PSD}} = \{\bmA\in \R^{n\times n} | \bmA \text{ is symmetric, } \bmA  \succeq  0\}$ is a proper cone, but its induced ordering
can not be used to define a lattice. 
	 There is a simple counter example to show this in \cref{app_sec_counter_psd}.}
	
Specifically, we name the lattice that can be defined through 
a conic inequality as ``conic lattice'', 
 since it is of particular interest for modeling  the real-world 
applications in this paper.
\begin{definition}[Conic Lattice] \label{def_conic_lattice}
Given a poset $(\X, \lleq_{\cone})$ induced by 
the conic inequality $\lleq_{\cone}$, if there exist  joint and meet operations  
for every pair of elements $(\a,\b)$ in $\X\times \X$, 
s.t. $\a \vee \b$ and $\a \wedge \b$ are still in $\X$, then 
we call $(\X, \lleq_{\cone})$ a conic lattice. 
\end{definition}
In one word, a conic lattice $(\X, \lleq_{\cone})$ is a lattice induced  by a conic inequality $\preceq_{\cone}$.
In the following we introduce a class of conic lattices to  model 
the applications in this work. We further provide  a general 
characterization about submodularity on this conic lattice. 

\textbf{Orthant conic lattice.} 
Given a sign vector $\bmalpha\in \{\pm1\}^n$, 
  the orthant cone is defined as  $\cone_{\bmalpha}:= \{\x\in \R^n \;|\; x_i\alpha_i \geq 0, \forall i\in [n]\}$, one can see that $\cone_{\bmalpha}$ is a proper cone.  
%
%
 For any two points $\a, \b\in \X $, 
 one can further  define the join and meet operations:
 $(\a\vee \b)_i := \alpha_i \max\{\alpha_ia_i, \alpha_ib_i \}$, $(\a\wedge \b)_i :=\alpha_i \min\{\alpha_ia_i, \alpha_ib_i \}$, $\forall i\in [n]$. Then one can show that  the poset $(\X, \lleq_{\cone_{\bmalpha}})$ is a conic  lattice. 

A function $f:\X\mapsto \R$  is submodular on a lattice \citep{topkis1978minimizing,fujishige2005submodular} if for all $(\x, \y)\in \X \times \X$, it holds that,
\begin{align}\label{eq_sub_lattice}
f(\x) + f(\y) \geq f(\x \vee \y)  + f(\x \wedge \y).
\end{align}

One can establish the  characterizations of submodularity 
on the orthant conic lattice $(\X, \preceq_{\cone_{\bmalpha}})$ similarly 
as that in \citet{bian2017guaranteed}: 
\begin{proposition}[Characterizations of Submodularity  on  Orthant Conic Lattice  $(\X, \preceq_{\cone_{\bmalpha}})$]\label{prop_submod_orthant}
If a function $f$ is is submodular on the lattice  $(\X, \preceq_{\cone_{\bmalpha}})$   (called \emph{$\cone_{\bmalpha}$-submodular}), then we have  the following two equivalent characterizations:\\
a)  $\forall \a,\b \in \X$ s.t. $\a \preceq_{\cone_{\bmalpha}} \b$, $\forall i$ s.t. $a_i =b_i$, 
 $\forall k\in \R_+$ s.t. $(k\e_i+\a)$ and $(k\e_i+\b)$ are still in $\X$, it holds that,
 	\begin{align}\label{eq_general_weakdr}
 \alpha_i [	f(k\e_i+\a) - f(\a)] \geq \alpha_i [f(k\e_i+\b) - f(\b)].  \quad \texttt{ \emph{(weak DR)} }
 	\end{align}
	
b) 	If $f$ is twice differentiable, then  $\forall \x\in \X$ it holds, 
	\begin{flalign}\label{eq_general_sub}
	\alpha_i\alpha_j\nabla_{ij}^2 f(\x) \leq 0,\; \forall i,j \in [n], \textcolor{blue}{i\neq j}.
	\end{flalign}
\end{proposition}
\cref{prop_submod_orthant} can be proved by directly 
generalizing the proof of Proposition 1    in \citet{bian2017guaranteed},  proof is omitted here due to the high similarity. Next, we generalize the definition of
DR-submodularity to the conic lattice $(\X, \preceq_{\cone_{\bmalpha}})$: 
\begin{definition}[$\cone_{\bmalpha}$-DR-submodular]
	\label{def_general_dr} 
	A  function $f:\X\mapsto \R$ is  $\cone_{\bmalpha}$-DR-submodular  if  $\forall \a, \b\in \X$ s.t. $\a \preceq_{\cone_{\bmalpha}}\b$, $\forall i \in [n], \forall k\in \R_+$ s.t. $(k\e_i+\a)$ and $(k\e_i+\b)$ are still in $\X$, it holds that,
	\begin{align}\label{eq_general_dr}
\alpha_i [	f(k\e_i+\a) - f(\a)] \geq \alpha_i [f(k\e_i+\b) - f(\b)].
	\end{align}
\end{definition}

In correspondence to the relation between DR-submodularity and
submodularity over continuous domains  (Proposition 2    in \citet{bian2017guaranteed}), one can easily get  the similar relation (with highly similar proof) in bellow: 
%
%
%
 \begin{proposition}[$\cone_{\bmalpha}$-submodular + coordinate-wise concave $\Leftrightarrow$ $\cone_{\bmalpha}$-DR-submodular]\label{prop_relation}
 A function $f$ is $\cone_{\bmalpha}$-DR-submodular iff it is $\cone_{\bmalpha}$-submodular and coordinate-wise concave. 
 \end{proposition}
 Combining \labelcref{eq_general_sub} and \cref{prop_relation},
 one can show that  if $f$ is twice differentiable and $\cone_{\bmalpha}$-DR-submodular,   then  $\forall \x\in \X$ it holds that, 
 	\begin{flalign}\label{eq_general_conedr}
 	\alpha_i\alpha_j\nabla_{ij}^2 f(\x) \leq 0,\; \forall  i, j\in  [n].
 	\end{flalign}
Similarly, a function $f$ is   $\cone_{\bmalpha}$-DR-supermodular iff $-f$ is $\cone_{\bmalpha}$-DR-submodular.
 
\textbf{Remark:} 
We only consider the orthant conic lattice $(\X, \preceq_{\cone_{\bmalpha}})$ here, since it can
already model the applications in this paper. However, it
is noteworthy that the framework can be generalized 
to arbitrary conic lattices, which may be of interest 
to model more complicated applications. We left this as future exploration.

 \subsection{A Reduction to Optimizing  Submodular Functions over Continuous Domains}

 To be succint, in this section we only discuss the reduction for 
 the $\cone_{\bmalpha}$-DR-submodular maximization problems. However, it is easy to see that  the reduction works
 for all  kinds of $\cone_{\bmalpha}$-submodular optimization problems,
 e.g., $\cone_{\bmalpha}$-submodular minimization problem.

Suppose $g$ is a  $\cone_{\bmalpha}$-DR-submodular 
function, and 
the $\cone_{\bmalpha}$-DR-submodular maximization problem  is $\max_{\y\in\P'} g(\y)$, where 
$\P' = \{ \y\in \R^n|h_i(\y)\leq b_i, \forall i\in [m], \y \ggeq_{\cone_{\bmalpha}} \mathbf{0} \}$ is down-closed 
w.r.t. the conic inequality $\lleq_{\cone_{\bmalpha}}$. 
The down-closedness here means if $\a\in \P'$ and 
$\mathbf{0}\lleq_{\cone_{\bmalpha}} \b \lleq_{\cone_{\bmalpha}} \a$, then $\b\in \P'$ as well.

Let $\bmA:=\diag(\bmalpha)$, and  a function $f(\x):= g(\bmA \x)$. One can see that 
if $g$ is $\cone_{\bmalpha}$-DR-submodular, then $f$ is DR-submodular:
assume wlog.\footnote{If not one can  still use  other  equivalent characterizations, for instance, the characterization in \labelcref{eq_sub_lattice} or in \labelcref{eq_general_weakdr} to formulate this.} that $g$ is twice differentiable, then  $\nabla^2f(\x) = \bmA^\trans \nabla^2 g \bmA$, and  $\nabla^2_{ij}f(\x) = \alpha_i\alpha_j \nabla^2_{ij} g \leq 0$, so $f$ is DR-submodular. 

By the affine
transformation $\y:=\bmA\x$, one can transform the 
$\cone_{\bmalpha}$-DR-submodular maximization problem  to be a DR-submodular maximization problem $\max_{\x\in\P} g(\bmA\x)$, where 
$\P = \{ \x\in \R^n|h_i(\bmA\x)\leq b_i, \forall i\in [m],  \bmA\x \ggeq_{\cone_{\bmalpha}} \mathbf{0} \}$ is down-closed 
w.r.t. the ordinary component-wise inequality $\leq$. 
To verify the down-closedness of $\P$ w.r.t. to the ordinary inequality  $\leq$ here, let 
$\y_1 = \bmA \x_1 \in \P'$ (so $\x_1\in \P$). Suppose there is a point
$\y_2 = \bmA \x_2$ s.t.  
$\mathbf{0}\lleq_{\cone_{\bmalpha}} \y_2 \lleq_{\cone_{\bmalpha}} \y_1$. From the down-closedness
of $\P'$, we know that $\y_2\in \P'$, thus $\x_2\in \P$. 
Looking at $\mathbf{0}\lleq_{\cone_{\bmalpha}} \y_2 \lleq_{\cone_{\bmalpha}} \y_1$, it is equivalent to that $0\leq \x_2\leq \x_1$. Thus we establish the down-closedness of $\P$.

Given the reduction, we can reuse the algorithms for  the original DR-submodular maximization problem \labelcref{setup}.


\subsection{Proof for the  Logistic Loss in \cref{subsec_moti_lattice}}
\label{appe_dr_lr}

Remember that the logistic loss is: 
\begin{flalign} \label{app_lr}
  l(\x) =\frac{1}{m}\sum\nolimits_{j=1}^{m}f_j(\x) = \frac{1}{m}\sum\nolimits_{j=1}^{m} \log(1 +\exp(-y_j \x^\trans \z^j))
\end{flalign}

\begin{claim}\label{claim_logistic}
$l(\x)$ in \labelcref{app_lr} is $\cone_{\bmalpha}$-DR-supermodular.
\end{claim}
\begin{proof}[Proof of \cref{claim_logistic}]
To show that $l(\x)$ is  $\cone_{\bmalpha}$-DR-supermodular, we can
check the second-order condition in \labelcref{eq_general_conedr}, that is, whether it holds that 
$\alpha_p\alpha_q\nabla_{pq}^2 l(\x) \geq  0,\; \forall p, q \in [n]$. One can easily see that, 
\begin{align}\notag 
& \fracpartial{l(\x)}{x_p} = \frac{1}{m}\sum\nolimits_{j=1}^{m} \frac{-y_j z_{p}^j}{\exp{(y_j \x^\trans \z^j)}+1}\\\notag 
& \fracppartial{l(\x)}{x_p}{x_q} =  \frac{1}{m}\sum\nolimits_{j=1}^{m} \frac{\exp{(y_j \x^\trans \z^j)}}{[\exp{(y_j \x^\trans \z^j)}+1]^2}z^j_p z^j_q .
\end{align}
Since $\alpha_p = \text{sign}(z^j_p)$, so $\alpha_p\alpha_q\nabla_{pq}^2 l(\x) \geq  0,\; \forall p, q \in [n]$. Thus 
$l(\x)$ in \labelcref{app_lr} is $\cone_{\bmalpha}$-DR-supermodular
according to \labelcref{eq_general_conedr}.
\end{proof}

%% file: appendix.tex

\section{More Applications}
\label{sec_more_apps}

We present more applications that 
fall into submodular or $\cone_{\bmalpha}$-submodular  optimization 
problems. 
One class of  notable examples are the objectives studied in \citet[]{NIPS2016_6073} in   the online setting.
%
These objectives are  captured by the  DR-submodular property
over continuous domains. One can also refer to
Section 2.2 in \citet{bach2015submodular}
to see more examples.

\textbf{DR-submodular quadratic functions.}  {Price optimization} with
continuous prices is a DR-submodular quadratic optimization problem
\citep{ito2016large}.
Another representative  class of  DR-submodular quadratic objectives
arises when computing the {stability number}  $s(G)$ of a graph $G= (V, E)$ \citep{motzkin1965maxima}, 
${s(G)}^{-1} = \min_{\x\in \Delta}\x^\trans (\bmA + \bmI)\x$,
where $\bmA$ is the adjacency matrix of the graph $G$, $\Delta$ is the standard simplex. 
This is an instance of a convex-constrained  DR-submodular maximization problem.


\textbf{Non-negative PCA (NN-PCA).}
NN-PCA  \citep{zass2007nonnegative,montanari2016non} is widely used as   alternative models
of PCA for dimension reduction, since its projection involves
only non-negative weights---a required property in fields
like economics, bioinformatics and computer vision. 
For a given set of $m$ data points $\z^j\in \R^n, j\in [m]$, NN-PCA aims to solve the following non-convex  optimization problem:
\begin{flalign}\label{nn_pca_append}
\min_{\|\x\|_2 \leq 1, \x\geq 0} f(\x) : =  -\frac{1}{2} \x^\trans \left(\sum\nolimits_{j=1}^{m} \z^j {\z^j}^\trans\right ) \x.
\end{flalign}
Let $\bmA= \sum\nolimits_{j=1}^{m} \z^j {\z^j}^\trans$,  
one can see that,
\begin{flalign}\notag 
&	A_{pp} = \sum\nolimits_{j=1}^{m}(z_p^j)^2 \geq 0,  \;	A_{pq} = \sum\nolimits_{j=1}^{m} z^j_p z^j_q = A_{qp}. 
\end{flalign}

Let us make the following weak assumption: 
	For one dimension/feature $i$, all
	the data points have the same sign, i.e., 
	$\sign{z^j_i}$ is the same for all $j \in [m]$ (which can
	be achieved by easily scaling if not).  
Now, by 
choosing the sign  vector  $\bmalpha\in \{\pm 1\}^n$
to be
$\alpha_p  = \sign{z^j_p}, \forall p\in [n]$,
one can easily verify that $A_{pq}\alpha_p\alpha_q \geq 0,  \forall p,q\in [n]$. Notice that $\nabla^2 f$ in  \labelcref{nn_pca_append}  is $-\bmA$, 
so it holds that    $\alpha_p\alpha_q\nabla^2_{pq} f \leq 0,  \forall p,q\in [n]$, thus $f(\x)$ is $\cone_{\bmalpha}$-DR-submodular
according to \labelcref{eq_general_conedr}.
Thus we can treat \labelcref{nn_pca_append} as a constrained 
$\cone_{\bmalpha}$-DR-submodular minimization problem. 


\textbf{Submodular spectral functions.}
As discussed by \citet{bach2015submodular}, 
submodular spectral functions \citep{friedland2013submodular} in  the following  form are DR-submodular, 
\begin{flalign}\label{eq_spectral}
f(\x) = \log\de{\sum\nolimits_{i=1}^{n}x_i \bmA_i}, \x\in \R_+^n,
\end{flalign}
where $\bmA_i$ are positive definite matrices. 
One can check the DR-submodularity of  $f(\x)$  by
checking its second-order-derivatives.

\section{The Subroutine Algorithm}
\label{sec_nonconvex_fw}

\IncMargin{1em}
\begin{algorithm}[htbp]
	\caption{\algname{Non-convex Frank-Wolfe} $(f, \P, K, \epsilon, \x^\pare{0})$\citep{lacoste2016convergence}}\label{nonconvex_fw}
	\KwIn{$\max_{\x \in \P} f(\x)$,
		$\P$:   {convex} set, $K$: number of iterations, $\epsilon$: stopping tolerance}
	\For{$k = 0, ... , K$}{
		{find $\v^\pare{k} \text{ s.t. } \dtp{\v^\pare{k}}{\nabla f(\x^\pare{k})} \geq   \max_{\v\in\P} \dtp{\v}{ \nabla f(\x^\pare{k})}$\tcp*{\emph{LMO}}}
		{$\d_k \leftarrow \v^\pare{k} - \x^\pare{k}$, $g_k := \dtp{\d_k}{\nabla f(\x^\pare{k})}$ \tcp*{$g_k$: non-stationarity measure}}
		{\bfseries{if $g_k \leq \epsilon$ then return  $\x^\pare{k}$}\;}
		{Option I:  $\gamma_k \in \argmin_{\gamma\in [0, 1]}f(\x^\pare{k} + \gamma \d^\pare{k})$, 
			Option II: $\gamma_k \leftarrow \min \{\frac{g_k}{C}, 1 \}$ for  $C\geq C_f(\P)$ \;}
		{$\x^\pare{k+1}\leftarrow \x^\pare{k} + \gamma_k \d^\pare{k}$ \;}
	}
	\KwOut{$\x^\pare{k'}$ and $g_{k'} = \min_{0\leq k\leq K} g_k$ \tcp*{modified  output solution compared to \citet{lacoste2016convergence}}}
\end{algorithm}
\DecMargin{1em}

\cref{nonconvex_fw} is taken from \citet{lacoste2016convergence}, the only
difference lies in the output: we output 
the solution $\x^\pare{k'}$
with the minimum non-stationarity, which is needed
to apply  the local-global relation. While \citet{lacoste2016convergence} outputs the solution 
in the last step. Since $C_f(\P)$ is generally hard to 
evaluate, we used the classical \algname{Frank-Wolfe}
step size $\frac{2}{k+2}$ in the experiments.

\section{Proofs for Properties}
\label{app_proofs_struc_algs}

\subsection{Proof of \cref{lemma_3_1}}

\begin{proof}[Proof of \cref{lemma_3_1}]
   Since $f$ is DR-submodular, so it is concave along any direction $\v\in \pm \R^n_+$. We know that $\x\vee \y - \x \geq 0$
    and $\x\wedge \y - \x\leq 0$, so from the strong DR-submodularity in \labelcref{eq_strong_dr}, 
    \begin{align}\notag 
   & f(\x\vee\y)  - f(\x) \leq \dtp{\nabla f(\x)}{\x\vee \y - \x}  -\frac{\mu}{2}  \|\x\vee \y - \x\|^2,\\\notag 
    &  f(\x\wedge \y) - f(\x) \leq \dtp{\nabla f(\x)}{\x\wedge \y - \x}  -\frac{\mu}{2}  \|\x\wedge  \y - \x\|^2.
    \end{align}
    Summing the above two inequalities and notice that $\x\vee\y + \x\wedge \y = \x+\y$, we arrive,
    \begin{flalign}\notag 
     (\y-\x)^{\trans}\nabla f(\x) & \geq f(\x\vee\y) + f(\x\wedge \y) - 2f(\x) + \frac{\mu}{2}  (\|\x\vee \y - \x\|^2 + \|\x\wedge  \y - \x\|^2)\\\notag 
     & = f(\x\vee\y) + f(\x\wedge \y) - 2f(\x) + \frac{\mu}{2}  \| \y - \x\|^2,
    \end{flalign}
   the last equality holds since $\|\x\vee \y - \x\|^2 + \|\x\wedge  \y - \x\|^2 =  \| \y - \x\|^2$.
\end{proof}

\subsection{Proof of  \cref{local_global}}\label{app_claim_proof}

\begin{proof}[Proof of \cref{local_global}]
Consider the point $\z^*:= \x\vee \x^* -\x = (\x^* - \x)\vee 0$. One can see that: 1) $0\leq \z^* \leq \x^*$; 2) $\z^* \in \P$ (down-closedness); 3) $\z^*\in \Q$ (because of   $\z^*\leq \bar \u - \x$). 
From \cref{lemma_3_1},
\begin{align}\label{eq_1718}
& \dtp{\x^*-\x}{\nabla f(\x)} +  2f(\x) \geq f(\x\vee \x^*) + f(\x \wedge \x^*) +  \frac{\mu}{2}\|\x -\x^*\|^2, \\\label{eq12}
& \dtp{\z^*-\z}{\nabla f(\z)} +  2f(\z) \geq f(\z\vee \z^*) + f(\z \wedge \z^*) +  \frac{\mu}{2}\|\z -\z^*\|^2.
\end{align}
Let us first of all prove the following  key \namecref{claim_key}.

\keyclaim* 

\begin{proof}[Proof of \cref{claim_key}]
Firstly, we are going to prove that 
\begin{align}\label{proof_part1}
f(\x \vee \x^*) + f(\z\vee \z^*) \geq f(\z^*) + f((\x+\z)\vee \x^*), 
\end{align}
which is equivalent to
$f(\x \vee \x^*) - f(\z^*) \geq f((\x+\z)\vee \x^*) - f(\z\vee \z^*)$.
It can be shown that  $\x \vee \x^*  - \z^* = (\x+\z)\vee \x^* - \z\vee \z^* $. Combining this with 
the fact that $\z^* \leq \z\vee \z^*$, and using the DR property (see \labelcref{eq_dr}) implies 
\labelcref{proof_part1}.
Then we establish,
\begin{align}\label{eq_EqaulityPoints}
 \x \vee \x^*  - \z^* = (\x+\z)\vee \x^* - \z\vee \z^* ~.
\end{align}
We will show that both the RHS and LHS of the above equation are equal to $\x$:  for the LHS of \labelcref{eq_EqaulityPoints} we can write 
 $\x \vee \x^*  - \z^* =  \x \vee \x^*  - \left(  \x \vee \x^* - \x\right) = \x$.
For the RHS of \labelcref{eq_EqaulityPoints} let us consider any coordinate $i\in [n]$,
\begin{align}\notag 
(x_i+z_i)\vee x_i^* - z_i\vee z_i^* = (x_i+z_i)\vee x_i^* - \left((x_i+z_i)-x_i\right)\vee  \left((x_i \vee x_i^*) - x_i\right) =x_i,
\end{align}
where the last equality holds easily for the two situations: $(x_i+z_i) \geq  x_i^*$ and $(x_i+z_i) < x_i^*$.

Next, we are going to prove that,
\begin{align}\label{proof_part2}
 f(\z^*) + f(\x\wedge \x^*)\geq f(\x^*) + f(0)
\end{align}
it is equivalent to 
$f(\z^*)   - f(0) \geq  f(\x^*) - f(\x\wedge \x^*)$,
which can be done similarly by the DR property: Notice that
\begin{align}\notag 
\x^* - \x\wedge \x^* = \x\vee \x^* - \x = \z^* - 0 \text{ and } 
 0 \leq  \x\wedge \x^*
\end{align}
thus \labelcref{proof_part2} holds from the DR property. 
Combining \labelcref{proof_part1,proof_part2} one can get,
\begin{align}\notag 
 f(\x \vee \x^*) + f(\z\vee \z^*) + f(\x\wedge \x^*) + f(\z\wedge \z^*)  & \geq  f(\x^*) + f(0) +  f((\x+\z)\vee \x^*)+ f(\z\wedge \z^*)\\\notag 
& \geq f(\x^*)    \quad \text{(non-negativity of $f$) }.
\end{align}
\end{proof} 

Combining \labelcref{eq_1718,eq12} and \cref{claim_key} it reads,
\begin{align}\label{eq16}
 \dtp{\x^* \!\!-\!\!\x}{\nabla f(\x)} \!\!+  \dtp{\z^*\!\!-\!\!\z}{\nabla f(\z)} \!\! +   2(f(\x) + f(\z) ) \geq f(\x^*) + 
 \frac{\mu}{2}(\|\x\!\! -\!\!\x^*\|^2 + \|\z \!\!-\!\!\z^*\|^2) 
\end{align}

From the definition of non-stationarity in \labelcref{non_stationary} one can get, 
\begin{align}\label{eq17}
&  g_{\P}(\x) := \max_{\v\in\P}\dtp{\v - \x}{\nabla f(\x)} \overset{\x^*\in \P}{\geq}  \dtp{\x^*-\x}{\nabla f(\x)}\\\label{eq18}
& g_{\Q}(\z) := \max_{\v\in\Q}\dtp{\v - \z}{\nabla f(\z)}  \overset{\z^*\in \Q}{\geq} \dtp{\z^*-\z}{\nabla f(\z)} 
\end{align}
Putting together \labelcref{eq16,eq17,eq18} we can get, 
\begin{align}\notag 
2(f(\x) + f(\z) ) \geq f(\x^*) -g_{\P}(\x) -g_{\Q}(\z) +   \frac{\mu}{2}(\|\x -\x^*\|^2 + \|\z -\z^*\|^2).
\end{align}
so it arrives 
$\max\{f(\x), f(\z) \} \geq \frac{1}{4}[f(\x^*) -g_{\P}(\x) -g_{\Q}(\z)]  +   \frac{\mu}{8}(\|\x -\x^*\|^2 + \|\z -\z^*\|^2)$.
\end{proof}

\section{Proofs for  Algorithms}
\label{app_proofs_algs}

\subsection{Proof of \cref{rate_local_fw}}

\begin{proof}[Proof of \cref{rate_local_fw}]
%
%

Let $g_{\P}(\x), g_{\Q}(\z)$ to be the non-stationarity of $\x$ and
$\z$, respectively. Since we are using 
the \algname{non-convex Frank-Wolfe} (\cref{nonconvex_fw}) as
subroutine, according to \citet[Theorem 1]{lacoste2016convergence}, one can  get,
	\begin{align}\notag 
&	g_{\P}(\x) \leq \min\left\{\frac{\max \{2h_1, C_f(\P)\}}{\sqrt{K_1+1}} , \epsilon_1\right\}   \\\notag
&	 g_{\Q}(\z) \leq  \min\left\{\frac{\max \{2h_2, C_f(\Q)\}}{\sqrt{K_2+1}} , \epsilon_2 \right\}, 
	\end{align}
	Plugging the above into \cref{local_global} we reach the  conclusion in \labelcref{eq_local_rates}.
\end{proof}

\subsection{Proof of \cref{prop_non_fw}}

\restalemmatwo*

\begin{proof}[Proof of \cref{prop_non_fw}]
	We prove  by induction. 	
	First of all, it holds when $k=0$, since $x_i^\pare{0}=0$,
	and $t^\pare{0}=0$ as well. 	
	Assume it holds for $k$. Then for $k+1$, we have
	\begin{align}\notag 
	x_i^\pare{k+1} & = x_i^\pare{k} + \gamma v_i^\pare{k}\\
	& \leq x_i^\pare{k} + \gamma ({\bar u_i} - x_i^\pare{k}) \quad \text{(constraint of shrunken LMO)}\\\notag 
	& = (1-\gamma) x_i^\pare{k} + \gamma {\bar u_i}\\
	& \leq (1-\gamma){\bar u_i}[1- (1-\gamma)^{t^\pare{k}/\gamma} ]+ \gamma {\bar u_i} \quad \text{ (induction) } \\\notag 
	& =  \bar u_i [1- (1-\gamma)^{t^\pare{k+1}/\gamma}].
	\end{align}
\end{proof}

\subsection{Proof of \cref{lem_nonmonotone_fw}}

\restalemmathree*

\begin{proof}[Proof of \cref{lem_nonmonotone_fw}]
	
	Consider $r(\lambda)= \x^* + \lambda(\x\vee \x^* - \x^*)$, it is easy to
	see that $r(\lambda)\geq 0, \forall \lambda \geq 0$. 
	
	Notice that $\lambda'\geq 1$. 
	Let $\y = \r(\lambda') =  \x^* + \lambda'(\x\vee \x^* - \x^*)$, it is easy to see that $\y \geq 0$, it also hold that $\y\leq \bar u$: Consider one coordinate $i$, 1) if $x_i\geq x_i^*$, then $y_i = x_i^* + \lambda'(x_i - x_i^*)\leq \lambda'x_i \leq \lambda'\theta_i \leq \bar u_i$; 2)  if $x_i< x_i^*$, then $y_i = x_i^* \leq \bar u_i$. So $f(\y) \geq 0$. 
	
	Note that 
	$$\x\vee \x^* = (1-\frac{1}{\lambda'})\x^* + \frac{1}{\lambda'}\y = (1-\frac{1}{\lambda'})r(0) + \frac{1}{\lambda'}r(\lambda'), $$
	since $f$ is concave along $r(\lambda)$, so it holds that,
	$$f(\x\vee \x^*) \geq  (1-\frac{1}{\lambda'})f(\x^*) +  \frac{1}{\lambda'}f(\y) \geq (1-\frac{1}{\lambda'})f(\x^*).$$
\end{proof}

\subsection{Proof of \cref{thm-e}}
\label{app_subsec_thm2_proof}

\begin{proof}[Proof of \cref{thm-e}]
	
First of all, let us prove the  \namecref{claim3_1}: 

\restaclaimthree*

	\begin{proof}[Proof of \cref{claim3_1}]
		Consider a point  $\z^\pare{k}:= \x^\pare{k}\vee \x^*  - \x^\pare{k}$, one can observe that:  1) $\z^\pare{k}\leq \bar \u -\x^\pare{k}$; 2) since $\x^\pare{k}\geq 0, \x^*\geq 0$, so $\z^\pare{k}\leq \x^*$, which implies that $\z^\pare{k}\in \P$ (from down-closedness of $\P$). So 
		$\z^\pare{k}$ is a candidate solution for the new  LMO  (Step \labelcref{new_lmo}). We have,
		\begin{flalign}\notag 
		f(\x^{\pare{k+1}}) - f(\x^{\pare{k}}) & \geq  \gamma\dtp{\nabla f(\x^\pare{k})}{\v^\pare{k}} - \frac{L}{2}\gamma^2 \|\v^\pare{k}\|^2 \quad (\text{Quadratic lower bound from \labelcref{eq_quad_lower_bound}}) \\\notag 
		&  \geq  \gamma\dtp{\nabla f(\x^\pare{k})}{\v^\pare{k}} - \frac{L}{2}\gamma^2 D^2 \quad (\text{diameter of } \P) \\\notag
		& \geq \gamma \dtp{\nabla f(\x^\pare{k})}{\z^\pare{k}} - \frac{L}{2}\gamma^2 D^2\quad (\text{shrunken LMO})\\ \notag
		& \geq \gamma(f(\x^\pare{k}+\z^\pare{k}) - f(\x^\pare{k}))  - \frac{L}{2}\gamma^2 D^2   \quad (\text{concave along $\z^\pare{k}$})\\\notag
		& = \gamma [f(\x^\pare{k}\vee \x^*)  - f(\x^\pare{k})]  - \frac{L}{2}\gamma^2 D^2\\\notag
		& \geq \gamma [(1-\frac{1}{\lambda'})f(\x^*)  - f(\x^\pare{k})]  - \frac{L}{2}\gamma^2 D^2 \quad (\text{\cref{lem_nonmonotone_fw}}) \\\notag
		& =\gamma  [ (1-\gamma)^{t^\pare{k}/\gamma} f(\x^*) - f(\x^\pare{k})]   - \frac{L}{2}\gamma^2 D^2 
		\end{flalign}
		where the last equality comes from  {setting } $\bmtheta : = \bar \u(1-(1-\gamma)^{t^\pare{k}/\gamma})$ { according to \cref{prop_non_fw}}, thus 
		$\lambda' = \min_i \frac{\bar u_i}{\theta_i} = (1-(1-\gamma)^{t^\pare{k}/\gamma})^{-1}$. 
		
		After rearrangement, we reach the claim. 
	\end{proof}
	Then, let us prove  \cref{thm-e} by \emph{induction}. 
	
	First of all, it holds when 
	$k = 0$ (notice that $t^\pare{0}=0$). Assume that it holds for $k$. Then for $k+1$, 
	considering  the fact $e^{-t} - O(\gamma)\leq (1-\gamma)^{t/\gamma}$ when 
	$0< \gamma\leq t \leq 1$ and \cref{claim3_1} we get, 
	\begin{align}\notag  
	& f(\x^{\pare{k+1}})\\\notag 
	&  \geq (1-\gamma)  f(\x^{\pare{k}})   + \gamma(1-\gamma)^{t^\pare{k}/\gamma} f(\x^*) -\frac{L D^2}{2}\gamma^2\\
	& \geq  (1-\gamma)  f(\x^{\pare{k}})   + \gamma [e^{-t^\pare{k}} - O(\gamma)] f(\x^*) -\frac{L D^2}{2}\gamma^2\\\notag 
	& \geq  (1-\gamma) [ t^\pare{k} e^{-t^\pare{k}}f(\x^*) - \frac{L D^2}{2}k\gamma^2 - O(\gamma^2)f(\x^*)]+ \gamma [e^{-t^\pare{k}} - O(\gamma)] f(\x^*) -\frac{L D^2}{2}\gamma^2\\\notag 
	& = [(1-\gamma) t^\pare{k} e^{-t^\pare{k}} + \gamma e^{-t^\pare{k}}   ]f(\x^*)  - \frac{L D^2}{2}\gamma^2 [(1-\gamma)k + 1] - [(1-\gamma) O(\gamma^2) + \gamma O(\gamma)]f(\x^*)\\\label{eq_30}
	& \geq  [(1-\gamma) t^\pare{k} e^{-t^\pare{k}} + \gamma e^{-t^\pare{k}}   ]f(\x^*) -  \frac{L D^2}{2}\gamma^2(k+1) - O(\gamma^2)f(\x^*).
	\end{align}
	Let us consider  the term $ [(1-\gamma) t^\pare{k} e^{-t^\pare{k}} + \gamma e^{-t^\pare{k}}   ]f(\x^*)$. We know that the function $g(t) = te^{-t}$ is concave in $[0, 2]$, 
	so $g(t^\pare{k}+\gamma) - g(t^\pare{k}) \leq \gamma g'(t^\pare{k})$, which amounts to
	\begin{align}\label{eq_34}
	[(1-\gamma) t^\pare{k} e^{-t^\pare{k}} + \gamma e^{-t^\pare{k}}   ]f(\x^*) \geq (t^\pare{k} +\gamma) e^{-(t^\pare{k} +\gamma)} f(\x^*) = t^{\pare{k+1}} e^{-t^{\pare{k+1}}} f(\x^*)
	\end{align}
	Plugging \labelcref{eq_34} into \labelcref{eq_30} we get, 
	\begin{align}\notag 
	f(\x^{\pare{k+1}})    \geq t^{\pare{k+1}} e^{-t^{\pare{k+1}}} f(\x^*) -  \frac{L D^2}{2}\gamma^2(k+1) - O(\gamma^2)f(\x^*).
	\end{align}
	Thus proving the induction, and proving the theorem as well. 
\end{proof}

\section{Miscellaneous Results}

\subsection{Verifying DR-submodularity of the Objectives}\label{appe_dr_soft}

\textbf{Softmax extension.}
For softmax extension, the objective is, 
\begin{flalign}\notag 
f(\x) = \log\de{\diag(\x)(\bmL-\bmI) +\bmI }, \x\in [0,1]^n.
\end{flalign}
Its DR-submodularity can be established by directly applying 
Lemma 3 in \citep{gillenwater2012near}:  \citet[Lemma 3]{gillenwater2012near} immediately implies 
that all  entries of  $\nabla^2  f$ are non-positive, so $f(\x)$
is DR-submodular. 

\textbf{Multilinear extension.}
The DR-submodularity of  multilinear extension can be directly
recognized by considering the conclusion in Appendix A.2 of \citet{bach2015submodular}
and the fact that multilinear extension is coordinate-wise linear.

\textbf{$\text{KL}(\x)$.}
The Kullback-Leibler divergence between  $q_{\x}$ and $p$, i.e., $ \sum_{S\subseteq \groundset} q_{\x}(S)
\log\frac{q_{\x}(S)}{p(S)}$ is, 
\begin{align}\notag 
\text{KL}(\x) = 
-\sum_{S\subseteq \groundset}\prod_{i\in S}x_i \prod_{j\notin S}(1-x_j) F(S) + \sum\nolimits_{i=1}^{n} [x_i\log x_i + (1-x_i)\log(1-x_i)] + \log Z.
\end{align}
The first term is the negative of a multilinear extension, so it is DR-supermodular. The second term
is separable, and coordinate-wise convex, so it will not
affect the off-diagonal entries of $\nabla^2 \text{KL}(\x)$,
it will only contribute to  the diagonal entries. 
Now, one can see that all entries of $\nabla^2 \text{KL}(\x)$  are non-negative, so $\text{KL}(\x)$ is DR-supermodular w.r.t. $\x$.

\subsection{A Counter Example to Show  PSD Cone is not Lattice}
\label{app_sec_counter_psd}

 The positive semidefine 
	cone $\cone_{\text{PSD}} = \{\bmA\in \R^{n\times n} | \bmA \text{ is symmetric, } \bmA  \succeq  0\}$ is a proper cone, but not 
	a lattice cone. That is, it can not be used to 
	define a lattice over the space of symmetric matrices. 
	
Let us consider the two dimensional symmetric  matrix space $S^2$.
Specifically, the following 
two symmetric matrices,
\begin{align}\notag 
\bmX = \begin{bmatrix}
 1 &0 \\
 0 & 0
\end{bmatrix}, 
\bmY = \begin{bmatrix}
 0 &0 \\
 0 & 1
\end{bmatrix}.
\end{align}
For the conic inequality $\lleq_{\cone_{\text{PSD}}}$, 
assume that there exists a least upper bound, i.e.,
the join of $\bmX, \bmY$: $\bmZ: = \bmX \vee \bmY$. From
the definition of least upper bound, $\forall\; \bmW \in S^2$ it should hold that,
\begin{align}\label{eq33}
\bmW   \ggeq_{\cone_{\text{PSD}}} \bmX \text{ and } \bmW   \ggeq_{\cone_{\text{PSD}}} \bmY  \text{ iff }  \bmW   \ggeq_{\cone_{\text{PSD}}} \bmZ.
\end{align}
Suppose $\bmZ = \begin{bmatrix}
 b &a \\
 a & c
 \end{bmatrix}$. Firstly, consider $\bmW$ to be diagonal 
  matrices, one can verify that $\bmZ$ must be in the form of 
$ \begin{bmatrix}
 1 &a \\
 a & 1
 \end{bmatrix}$, then considering $\bmW =\bmI$ forcing $\bmZ$
 to be $\bmI$. 
 
 Now let  $\bmW =\frac{2}{3} \begin{bmatrix}
  2 &1 \\
  1 & 2
  \end{bmatrix}$, which is $\ggeq_{\cone_{\text{PSD}}} \bmX$
  and $\ggeq_{\cone_{\text{PSD}}} \bmY$. However,  $\bmW - \bmI =\frac{1}{3} \begin{bmatrix}
    1 &2 \\
    2 & 1
    \end{bmatrix} \notin \cone_{\text{PSD}}$, thus
  contradicting \labelcref{eq33}.

\section{Additional Experimental Results}
\label{app_add_exp}

We generate the down-closed polytope constraints in the same form  
 and same way as that for  DR-submodular quadratic functions.

 \cref{fig_softmax_exp} shows the function values returned by different solvers w.r.t. $n$, for which the random polytope
constraints were generated with exponential distribution.
Specifically, the random polytope is in the form of   $\P = \{\x\in \R_+^n \ |\  \bmA \x \leq \b, \x \leq \bar \u, \bmA\in \R_{++}^{m\times n}, \b\in \R_+^m \}$. 
Each entry of $\bmA$ was
sampled from $\text{Exp}(1) + \nu$, where $\nu = 0.01$
is a small positive constant.
We set   $\b = 2*\mathbf{1}^m$, and set 
$\bar \u$ to be the tightest upper bound of $\P$ by  $\bar u_j = \min_{i\in [m] }\frac{b_i}{A_{ij}}, \forall j\in [n]$. 
One can see that the  \algname{two-phase Frank-Wolfe}
 has the best performance, while non-monotone
\algname{Frank-Wolfe}  and   \algname{ProjGrad} have comparable  performance.

  \setkeys{Gin}{width=0.33\textwidth}
         \begin{figure}[htbp]
       \center 
      \includegraphics[width=0.58\textwidth]{legend_h.pdf}\\
      \vspace{-0.4cm}
              \subfloat[$m={\floor {0.5n}}$ \label{fig_softmax_exp1}]{
              \includegraphics[]{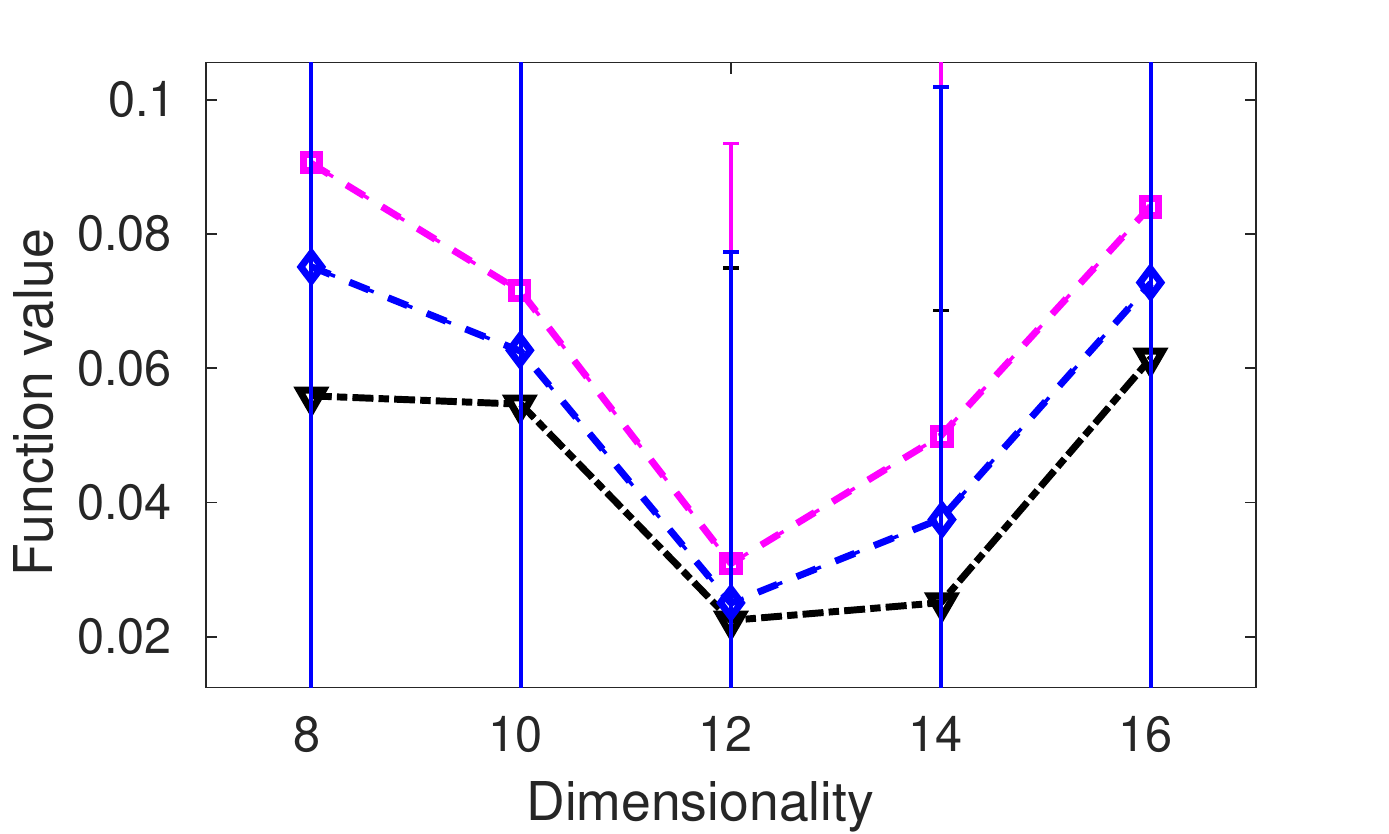}}
        \subfloat[$m=n$ \label{fig_softmax_exp2}]{
        \includegraphics[]{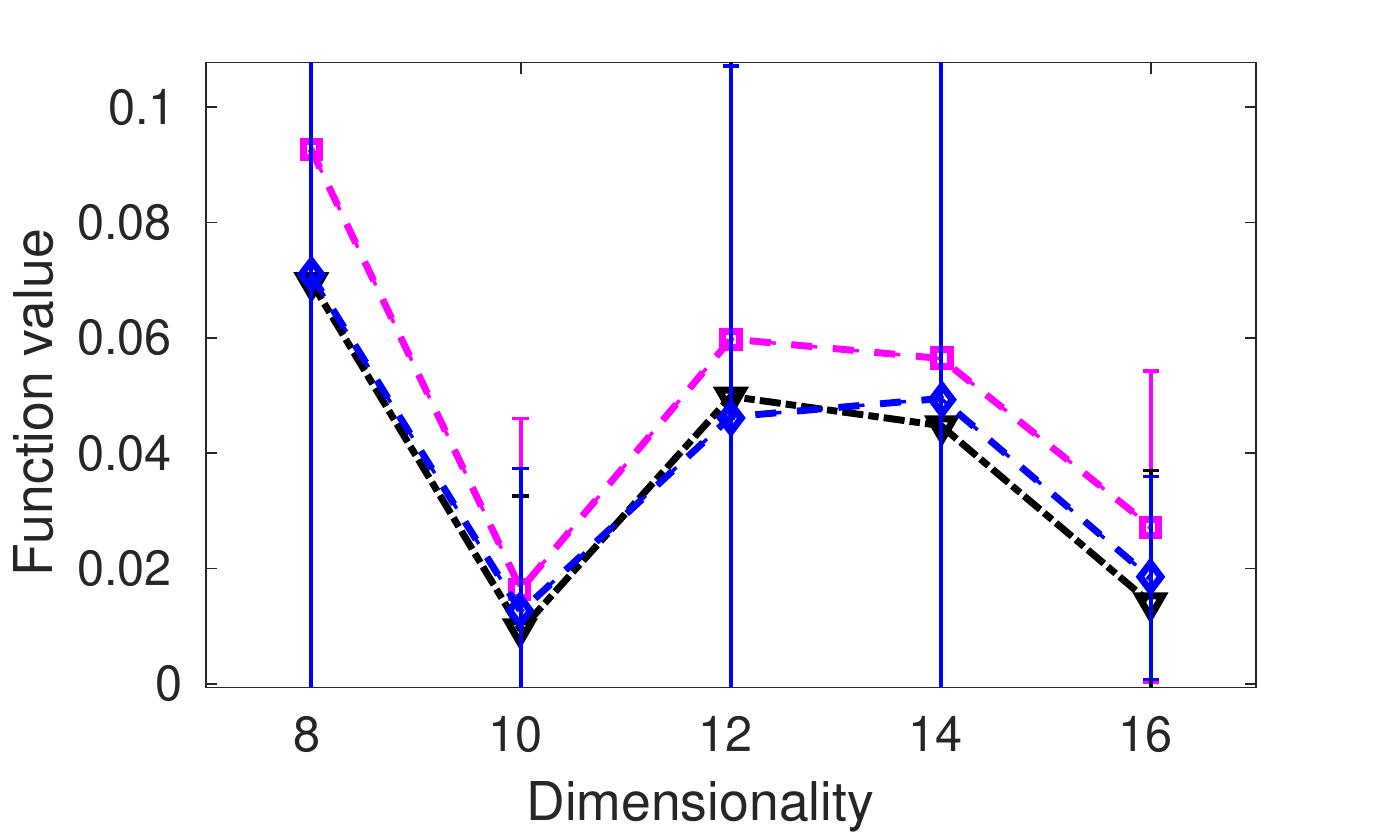}}
      \subfloat[$m=\floor {1.5n}$ \label{fig_softmax_exp3}]{
      \includegraphics[]{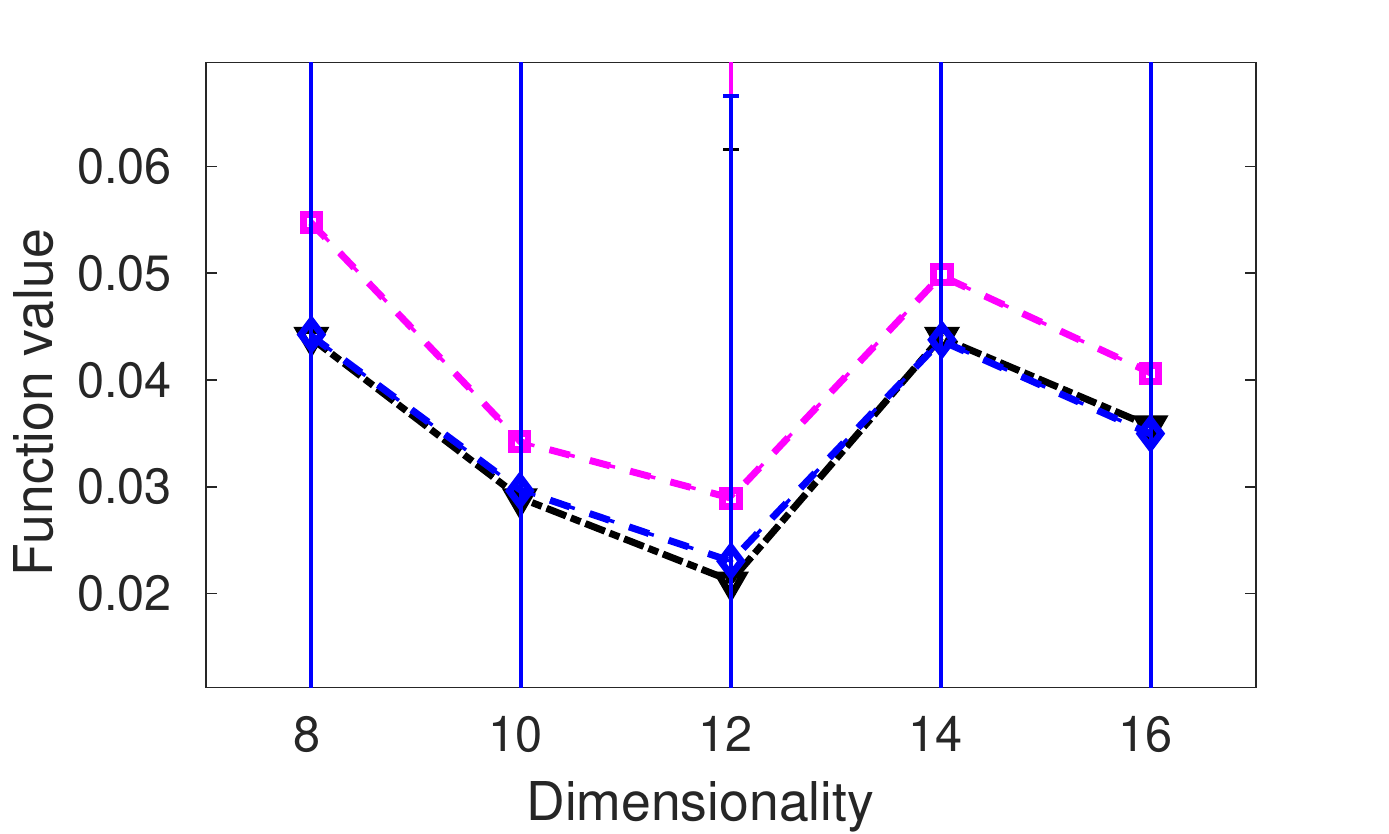}}
    \caption{Results  on   softmax instances with random polytope constraints generated  from exponential distribution.}
         \label{fig_softmax_exp}
    \end{figure}